\documentclass{article} % For LaTeX2e
\usepackage{iclr2027_conference,times}
\usepackage{amsmath,amsfonts,bm}

\def\eqref#1{equation~\ref{#1}}
\def\1{\bm{1}}

\DeclareMathAlphabet{\mathsfit}{\encodingdefault}{\sfdefault}{m}{sl}
\SetMathAlphabet{\mathsfit}{bold}{\encodingdefault}{\sfdefault}{bx}{n}

\usepackage{hyperref}
\usepackage{url}
\usepackage{amsthm}
\usepackage{algorithmic} 
\usepackage{enumitem}
\usepackage{placeins}
\usepackage{booktabs}
\usepackage{graphicx}
\usepackage{amssymb}
\usepackage{caption}
\usepackage{subcaption}
\usepackage[ruled]{algorithm2e}
\usepackage{titletoc}
\usepackage{natbib}
\usepackage{xcolor}

\newtheorem{theorem}{Theorem}
\newtheorem{proposition}[theorem]{Proposition}
\newtheorem{corollary}[theorem]{Corollary}
\newtheorem{assumption}{Assumption}

\newtheorem{lemma}{Lemma}

\ifodd 1
\else

\fi

\title{CFLoRA: Federated Fine-tuning of LLMs with Complementary Factors for Error-free Aggregation}

\author{Yanan Ma \\
City University of Hong Kong\\
\texttt{yananma8-c@my.cityu.edu.hk} 
\And
Qiyuan Chen \\
The University of Hong Kong \\
\texttt{qiyuanchen@connect.hku.hk} 
\And
Zihan Fang \\
City University of Hong Kong\\
\texttt{zihanfang3-c@my.cityu.edu.hk} 
\And
Xianhao Chen \\
The University of Hong Kong \\
\texttt{xchen@eee.hku.hk} 
\And
Yuguang Fang \\
City University of Hong Kong\\
\texttt{my.fang@cityu.edu.hk}
}

\iclrfinalcopy % Uncomment for camera-ready version, but NOT for submission.
\begin{document}

\maketitle

\begin{abstract}
Federated low-rank adaptation (LoRA) enables collaborative fine-tuning of large language models without centralizing private client data. Its factorized update, however, creates a structural mismatch in federated averaging: averaging the two LoRA factors separately does not equal averaging their products. Existing exact methods resolve this issue mainly by freezing an entire factor or alternating factors across rounds, but none can update factors simultaneously without aggregation errors or expanding communication ranks. To address this fundamental problem, we present \texttt{CFLoRA}, a federated LoRA scheme that partitions latent LoRA channels into two complementary sets in every communication round. By ensuring that columns and rows are complementary across factors, we eliminate bilinear terms in matrix multiplications, making federated aggregation exact. Crucially, our framework also supports clients with heterogeneous rank budgets. Convergence analysis validates \texttt{CFLoRA} achieves $\mathcal{O}(1/\sqrt{T})$ convergence rate of the \textit{original} LoRA objective in homogeneous-rank cases. 
% Extensive experiments validate that \texttt{CFLoRA} achieves exact superior performance and training efficiency compared to state-of-the-art federated LoRA baselines.
Extensive experiments with RoBERTa on the GLUE benchmark and with LLaMA-3.2-3B-Instruct on commonsense reasoning tasks demonstrate that \texttt{CFLoRA} achieves superior performance and training efficiency compared to state-of-the-art federated LoRA baselines.
\end{abstract}

% \section{Introduction}
% \label{sec:introduction}

\section{Introduction}
\label{sec:introduction}

%Adapting large language models (LLMs) to private data . 
Federated learning (FL) enables clients to learn from private, distributed data without transferring raw data to a central server \citep{mcmahan2017communication,kairouz2021advances}, which is important for privacy-sensitive applications. However, full-parameter fine-tuning imposes substantial memory, computation, and communication costs on participating clients. Low-rank adaptation (LoRA) reduces these costs by freezing pretrained weights and representing each trainable weight update as the product of two matrices \citep{hu2022lora}. Federated LoRA, which integrates FL with LoRA, therefore offers a practical approach to distributed LLM adaptation subject to clients' resource constraints~\citep{zhang2024federatedgpt}.

Despite these benefits, it is well-known that LoRA's factorized parameterization introduces a fundamental mismatch in federated aggregation. For client factors $(\mathbf{B}_i,\mathbf{A}_i)$ and normalized aggregation weights $w_i$, in general,
\begin{equation}
  \left(\sum_i w_i\mathbf{B}_i\right)
\left(\sum_i w_i\mathbf{A}_i\right)
\neq
\sum_i w_i\mathbf{B}_i\mathbf{A}_i.  
\end{equation}
%Averaging the factors separately introduces cross-client products that are absent from the desired average of client adapters. 
The resulting aggregation error can impair federated optimization, particularly when heterogeneous data drives client factors apart \citep{sun2024ffa}. A detailed formulation and analysis is provided in Section \ref{subsec:aggregation_error}.

Existing approaches navigate this mismatch through various trade-offs. Freezing one factor throughout training or alternating entire factors across rounds ensures exact aggregation, but prevents simultaneous adaptation within a single round \citep{sun2024ffa,chen2025rolora,koo2025loraa2}. Restricting updates in this way limits the model's expressive capacity and slows convergence, as both matrices are supposed to co-adapt to capture knowledge. Alternatively, stacking client adapters preserves exact aggregation but undesirably expands the rank to the sum of all client ranks~\citep{wang2024flora}. Other methods adopt approximate decompositions~\citep{bai2024flexlora,bian2025lorafair,wang2026iflora}, incorporate residual corrections into the pretrained weights~\citep{singhal2025fedex}, or utilize alternative parameterizations with fixed outer bases~\citep{raje2025ravan}. However, these either only approximate or diverge from the original LoRA learning objective. These limitations raise a still unresolved question: \emph{Can two-factor federated LoRA achieve exact factor-wise aggregation and distribution while simultaneously updating both factors?}

We answer this question with \emph{complementary-channel federated LoRA} (\texttt{CFLoRA}). Our central insight is that exact aggregation can be enforced by a simple modification to LoRA: in every round, if a column of $\mathbf B$ is trainable, its corresponding row in $\mathbf A$ can be frozen, and vice versa. This mutually exclusive update ensures $\Delta\mathbf B_i\Delta\mathbf A_j=\mathbf 0$ for every client pair $(i,j)$, thus leading to exact factor-wise aggregation. \texttt{CFLoRA} can be treated as a generalized version of alternative factor updates, but enables simultaneous factor updating with flexible choices of trainable channels. Compared with standard federated LoRA, our scheme \textit{halves trainable parameters} and \textit{lowers communication costs}. More importantly, benefiting from exact aggregation and the co-adaptation of both factors, we show that \texttt{CFLoRA} achieves an $\mathcal{O}(1/\sqrt{T})$ convergence rate for the \textit{original} LoRA objective under homogeneous client ranks and a surrogate objective under heterogeneous ranks.

%We answer this question with \emph{complementary-channel federated LoRA} (\texttt{CFLoRA}). Our central insight is that exact aggregation can be enforced by a simple modification to LoRA: in every round, if a column of $\mathbf B$ is trainable, its corresponding row in $\mathbf A$ is frozen, and vice versa. This mutually exclusive update ensures $\Delta\mathbf B_i\Delta\mathbf A_j=\mathbf 0$ for every client pair $(i,j)$, thus leading to exact factor-wise aggregation. \texttt{CFLoRA} can be viewed as a generalization of alternating factor updates, but crucially, it enables simultaneous factor adaptation with flexible choices of trainable channels. Compared with standard federated LoRA, our scheme \textit{halves trainable parameters} and \textit{lowers communication costs} while adding almost no additional complexity (except for column/row localization). Benefiting from exact aggregation and the co-adaptation of both factors, \texttt{CFLoRA} achieves an $\mathcal{O}(1/\sqrt{T})$ convergence rate in both settings—optimizing the \textit{original} LoRA objective under homogeneous client ranks, and a well-defined \textit{surrogate} objective under heterogeneous ranks.

Our main contributions are as follows:
\begin{itemize}
    \item We introduce \texttt{CFLoRA}, a federated LoRA fine-tuning scheme with complementary factors that ensure exact factor-wise aggregation and seamlessly support heterogeneous ranks.

    \item 
    We establish an $\mathcal{O}(1/\sqrt{T})$ convergence rate for \texttt{CFLoRA} under both homogeneous and heterogeneous rank settings, and theoretically reveals the importance of factor co-adaptation.

    \item 
    We demonstrate the performance superiority and resource efficiency of \texttt{CFLoRA} over other federated LoRA baselines by fine-tuning RoBERTa and LLaMA-3.2-3B-Instruct models on GLUE and commonsense reasoning tasks, respectively.  
\end{itemize}

%We conduct extensive experiments across multiple datasets and LLM models with diverse parameter settings, demonstrating \texttt{CFLoRA}’s superior performance compared to various baselines in accuracy and resource utilization.
%Due to this design
% While the server maintains a global adapter of rank $R$, client $i$ can receive and activate only a subset of $r_i$ channels from the shared complementary assignments.

\section{Related Work}
\label{sec:related}

\paragraph{Federated LoRA Fine-tuning.} Research on Federated LoRA investigated how to preserve the parameter efficiency of low-rank adaptation while aggregating updates learned from decentralized data. FedIT directly combines LoRA with federated averaging and communicates both factors \citep{zhang2024federatedgpt}; this simple baseline produces an error between the product of factor averages and the desired average of products. FedEx-LoRA computes the aggregation error and absorbs this residual into the nominally frozen backbone \citep{singhal2025fedex}. However, this residual has a rank proportional to the aggregated rank and changes the base model across rounds. LoRA-FAIR optimizes a correction to an averaged factor against the product-space target, whereas iFLoRA reconstructs local products and applies SVD in a pipelined aggregation procedure \citep{bian2025lorafair,wang2026iflora}. However, these schemes introduce server-side reconstruction complexity (e.g., SVD of high-dimensional models) without entirely avoiding approximation errors.

Beyond error correction, another line of work modifies the factor-updating rule. FFA-LoRA fixes the randomly initialized $\mathbf A$ and updates only the zero-initialized $\mathbf B$~ \citep{sun2024ffa}. While the aggregation is exact, it introduces a permanently frozen factor that restricts the learned subspace. FedSA-LoRA instead shares only the factor associated with general knowledge and retains the other factor locally for personalization \citep{guo2025selective}. RoLoRA alternates the updated factor across communication rounds~\citep{chen2025rolora}. AS-LoRA \citep{kim2026adaptive} extends the alternating-update idea to the layer level, allowing each layer to select one factor to update at a time. These alternating schemes eliminate aggregation errors, yet still fail to update $\mathbf A$ and $\mathbf B$ concurrently. Consequently, none of the aforementioned works can \textit{simultaneously update both factors within the same round with exact factor-wise aggregation and distribution.}

\paragraph{Federated LoRA with Heterogeneous Ranks.}
Rank-heterogeneous federated LoRA assigns client-specific adapter capacity to accommodate differences in memory, computation, communication, and task complexity~\citep{fang2025automated}. Several approaches address this through padding or sampling mechanisms~\citep{cho2024hetlora,byun2025rank}. Alternatively, FSLoRA samples submatrices from a server-side LoRA model based on client-specific sketching ratios and provides a convergence analysis~\citep{fang2026fslora}. Another line of work relies on product-space decomposition. FlexLoRA constructs a weighted product-space aggregate and employs SVD to distribute rank-specific adapters to clients~\citep{bai2024flexlora}. FedPA-LoRA applies SVD solely on a small core matrix obtained via reduced QR factorization~\citep{jeon2026fedpa}. Nevertheless, repeated model decompositions incur significant overhead, and crucially, these schemes fail to preserve exact factor-wise aggregation. To preserve exact aggregation, FLoRA stacks heterogeneous client factors so that their product equals the sum of local products \citep{wang2024flora}. 
GLoRA aggregates factors in consensus coordinates~\citep{chen2026glora}. However, to eliminate aggregation error, FLoRA requires a broadcast rank of $\sum_i r_i$, and GLoRA also requires the server rank to cover the union span across clients. Other works achieve exact aggregation, yet with alternative updates or different parameterizations (e.g., updating only middle matrices)~\citep{raje2025ravan,koo2025loraa2,meng2026florg}. In a nutshell,  none of the existing FedLoRA methods, which update $\mathbf A$ and $\mathbf B$ concurrently, can support heterogeneous client ranks with exact factor-wise aggregation and distributions.

\section{Preliminaries and Problem Formulation}
\label{sec:preliminaries}

\subsection{Federated LoRA Fine-Tuning}
We consider a federated LoRA fine-tuning system comprising $N$ clients, which can be formulated as follows:
\begin{equation}
F(\mathbf{B}, \mathbf{A}) = \sum_{i} w_i f_i(\mathbf{B}, \mathbf{A}), \qquad
f_i(\mathbf{B}, \mathbf{A}) = \mathbb{E}_{\xi \sim \mathcal{D}_i}\!\left[\ell(\mathbf{W}_0 + \mathbf{B}\mathbf{A}; \xi)\right],
\label{eq:objective}
\end{equation}
where $\ell$, $f_i$, and $F$ denote the sample loss, the local loss for client $i$, and the global loss, respectively. $\xi$ represents a data sample drawn from the local dataset $\mathcal{D}_i$ on client $i$, and $w_i > 0$ is the weight assigned to client $i$ with $\sum_{i} w_i = 1$. Furthermore, $\mathbf{W}_0 \in \mathbb{R}^{d_o \times d_i}$ denotes the frozen base model, while $\mathbf{B} \in \mathbb{R}^{d_o \times R}$ and $\mathbf{A} \in \mathbb{R}^{R \times d_i}$ are the trainable rank-$R$ LoRA modules. 
% Because this objective aligns with standard federated optimization frameworks, it can be readily solved using the FedAvg algorithm.

\subsection{Failures in Factor-Wise Aggregation}
\label{subsec:aggregation_error}
A fundamental challenge in LoRA-based federated fine-tuning is that averaging and distributing $\mathbf{A}$ and $\mathbf{B}$ \textit{separately} following FedAvg \citep{mcmahan2017communication,sun2024ffa, chen2025rolora} is not equal to averaging and distributing their products. Formally, suppose $\mathbf{B}$ and $\mathbf{A}$ denote the global factors at the beginning of the round. Each client $i$ produces local endpoints 
$\mathbf{B}_i = \mathbf{B} + \Delta\mathbf{B}_i$, $\mathbf{A}_i = \mathbf{A} + \Delta\mathbf{A}_i$. Factor-wise averaging across $N$ clients yields the update:
\begin{equation}
\begin{aligned}
&\left(\mathbf{B} + \sum_{i}w_i\Delta\mathbf{B}_i\right)\left(\mathbf{A} + \sum_{i}w_i\Delta\mathbf{A}_i\right) \\
= &\mathbf{B}\mathbf{A} + \sum_{i}w_i\Delta\mathbf{B}_i \mathbf{A} + \mathbf{B} \sum_{i}w_i\Delta\mathbf{A}_i +\underbrace{\left(\sum_{i}w_i\Delta\mathbf{B}_i\right)\left(\sum_{i}w_i\Delta\mathbf{A}_i\right)}_{\text{cross-client term}}, \label{eq:factoravg}
\end{aligned}  
\end{equation}
whereas the ideal aggregation target is:
\begin{equation}
\begin{aligned}
\sum_{i} w_i \mathbf{B}_i\mathbf{A}_i &= \mathbf{B}\mathbf{A} + \sum_{i} w_i \Delta\mathbf{B}_i \mathbf{A} + \mathbf{B} \sum_{i} w_i \Delta\mathbf{A}_i + \underbrace{\sum_{i} w_i \Delta\mathbf{B}_i\Delta\mathbf{A}_i.}_{\text{client-wise interaction term}} \label{eq:modelavg}
\end{aligned}
\end{equation}
The discrepancy between Equations~\ref{eq:factoravg} and \ref{eq:modelavg} lies in the final terms, where the difference defines the \textit{aggregation error} of factor-wise averaging. 

This aggregation error compounds over successive rounds and degrades convergence. Its magnitude is governed by two key factors. First, it may grow with the \textit{number of local update steps}: more local steps can cause $\Delta\mathbf{B}_i$ and $\Delta\mathbf{A}_i$ to diverge further from their global mean. Second, it grows with the degree of \textit{data heterogeneity} across clients: when local datasets follow substantially different distributions. Taken together, these two factors motivate the design of new federated LoRA schemes.

\begin{figure*}[t]
\centering
\includegraphics[width=0.98\textwidth]{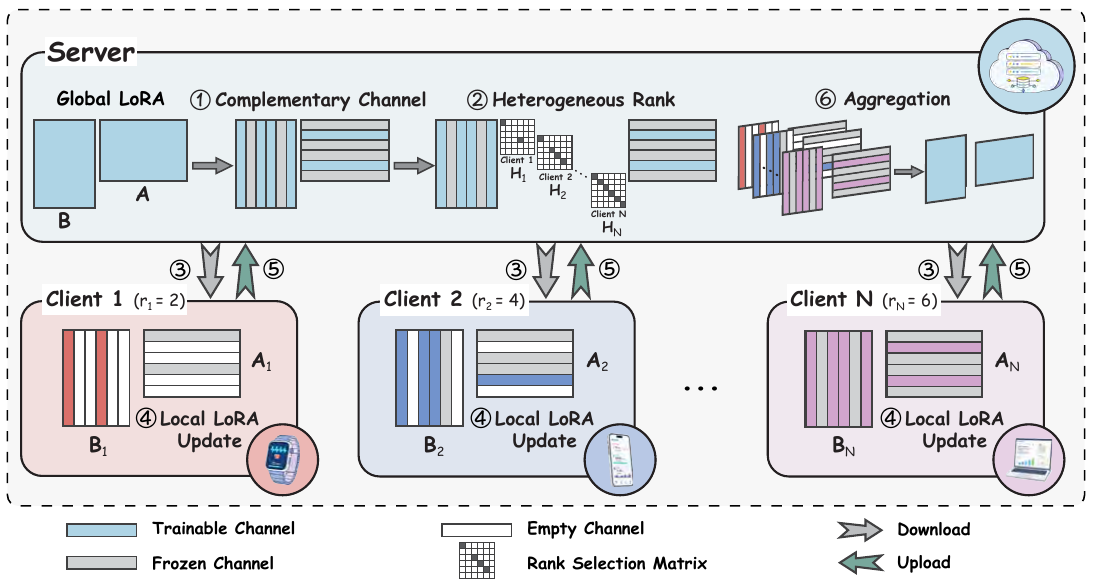}
\caption{The illustration of the proposed \texttt{CFLoRA} framework. \texttt{CFLoRA} separates two decisions: which factor may change on each latent channel and which channels a client can afford to train. A shared complementary selection mask ensures exact factor aggregation, while client-specific supports accommodate heterogeneous ranks.}
\label{fig:CFLoRA}
\end{figure*}

\section{The \texttt{CFLoRA} Framework}
\label{sec:framework}
% In this section, we present the framework design of \texttt{CFLoRA}, readily extend it to the heterogeneous rank case, and conduct convergence analysis.

\subsection{Complementary-Channel Selection}
\label{subsec:homogeneous}
The aggregation inconsistency discussed in Section~\ref{subsec:aggregation_error} arises from the bilinear interaction term $\Delta \mathbf{B}_i \Delta \mathbf{A}_i$. To eliminate this, we propose a \textit{complementary channel partitioning} strategy. In every round, if a column of $\mathbf  B$ is trainable, its corresponding row in $\mathbf  A$ can be frozen, and vice versa\footnote{Here, we describe a single LoRA layer for notational simplicity. For multiple layers, \texttt{CFLoRA} applies the same rule to each layer, with its complementary mask shared by participating clients.}.
Formally, at round $t$, the server constructs a diagonal channel-selection mask, defined as 
\begin{equation}
    \mathbf{Z}^t = \operatorname{diag}(z^t_1, \dots, z^t_R), \qquad z^t_j \in \{0, 1\},
    \label{eq:diag_z}
\end{equation}
where $z^t_j$ is sampled independently with $\mathrm{Pr}(z^t_j = 1) = p$. 
For each channel $j$, the update rule is defined as follows.
\begin{itemize}
    \item[-] If $z_j = 1$: update column $\mathbf{B}({:,j})$ and freeze row $\mathbf{A}({j,:})$;
    \item[-] If $z_j = 0$: freeze column $\mathbf{B}({:,j})$ and update row $\mathbf{A}({j,:})$.
\end{itemize}
The server broadcasts the same mask to all participating clients in each round, thereby ensuring a common update structure within each round while allowing the channel partition to vary across rounds.

During local training, client \(i\) initializes as \(\mathbf{B}_i^{t,0} = \mathbf{B}^t\) and \(\mathbf{A}_i^{t,0} = \mathbf{A}^t\) and updates at local step \(k\) according to $
\mathbf{B}_i^{t,k+1} = \mathbf{B}_i^{t,k} + \Delta\mathbf{B}_i^{t,k} \mathbf{Z}^t$ and  $\mathbf{A}_i^{t,k+1} = \mathbf{A}_i^{t,k} + (\mathbf{I}_R - \mathbf{Z}^t)\Delta\mathbf{A}_i^{t,k},$
where \(\Delta\mathbf{B}_i^{t,k}\) and \(\Delta\mathbf{A}_i^{t,k}\) denote client \(i\)'s unmasked local updates at step \(k\). Note that this full-matrix formulation is introduced primarily for clarity; in practice, clients can update only the active parameters determined by \(\mathbf{Z}^t\). After \(K\) local steps, the endpoint of client $i$ is $ \mathbf{B}_i^{t+1} = \mathbf{B}^t + \Delta\mathbf{B}_i^t \mathbf{Z}^t$ and $\mathbf{A}_i^{t+1} = \mathbf{A}^t + (\mathbf{I}_R - \mathbf{Z}^t)\Delta\mathbf{A}_i^t$, where the accumulated unmasked local updates over the $K$ steps are given by
\begin{equation}
\label{eq: homogeneous_updates}
\Delta\mathbf{B}_i^t = \sum_{k=0}^{K-1} \Delta\mathbf{B}_i^{t,k}, 
\qquad
\Delta\mathbf{A}_i^t = \sum_{k=0}^{K-1} \Delta\mathbf{A}_i^{t,k}.
\end{equation}
% are the accumulated unmasked local updates over the $K$ steps.

In each round, suppose that the server randomly samples a subset $\mathcal{S}_t$ of $m$ clients from the total $N$ clients\footnote{We consider uniform sampling rate $q$ for notational simplicity. Our analysis can be easily extended to consider client-specific sampling probability.}. Then, the server performs factor-wise aggregation:
\begin{equation}
\begin{aligned}
\mathbf B^{t+1}
=\mathbf B^t+\sum_{i} 
\frac{w_i}{q}\Delta\mathbf B_i^t\mathbf Z^t,
\qquad
\mathbf A^{t+1}
=\mathbf A^t+\sum_{i} 
\frac{w_i}{q}(\mathbf I_R-\mathbf Z^t)\Delta\mathbf A_i^t.
\end{aligned}
\label{eq:s4-aggregation}
\end{equation}
where $q = m/N$ denotes the client participation ratio.
The scaling $1/q$ compensates for the partial participation, ensuring an unbiased estimator of the full-client update. Since the complementary masks satisfy $\mathbf Z^t(\mathbf I_R-\mathbf Z^t)=\mathbf 0$, we can establish Proposition \ref{prop:exact} for the aggregation exactness.

% \begin{proposition}[Error-Free Aggregation of \texttt{CFLoRA}]
% \label{prop:exact}
% Factor-wise aggregation in \eqref{eq:s4-aggregation} is equivalent to aggregating the client-level LoRA matrices:
% \begin{equation}
% \left(\mathbf B^t+\sum_{i} 
% \frac{w_i}{q}\Delta\mathbf B_i^t\mathbf Z^t\right)\left(\mathbf A^t+\sum_{i} 
% \frac{w_i}{q}(\mathbf I_R-\mathbf Z^t)\Delta\mathbf A_i^t\right) = \sum_{i}
% \frac{w_i}{q}\mathbf B_i^{t+1}\mathbf A_i^{t+1}. \label{eq:exact_delta_average}
% \end{equation}
% \end{proposition}

\begin{proposition}[Error-Free Aggregation of \texttt{CFLoRA}]
\label{prop:exact}
Factor-wise aggregation in \eqref{eq:s4-aggregation} satisfies
\begin{equation}
\begin{aligned}
&\left(
\mathbf{B}^{t}
+\sum_{i}\frac{w_i}{q}
\Delta\mathbf{B}_i^{t}\mathbf{Z}^{t}
\right)
\left(
\mathbf{A}^{t}
+\sum_{i}\frac{w_i}{q}
(\mathbf{I}_R-\mathbf{Z}^{t})\Delta\mathbf{A}_i^{t}
\right)
\\
&\qquad =
\mathbf{B}^{t}\mathbf{A}^{t}
+\sum_{i}\frac{w_i}{q}
\left(
\mathbf{B}_i^{t+1}\mathbf{A}_i^{t+1}
-\mathbf{B}^{t}\mathbf{A}^{t}
\right).
\end{aligned}
\label{eq:exact_delta_average}
\end{equation}
\end{proposition}

Proposition~\ref{prop:exact} shows that \texttt{CFLoRA} achieves exact federated LoRA aggregation. The detailed proof is provided in Appendix~\ref{app:proof_exact_agg}.

% ------------------- Heterogeneous Client Ranks ----------------------
\subsection{Generalization to Heterogeneous Client Ranks}
\label{subsec:heterogeneous}
As shown in Fig. \ref{fig:CFLoRA}, the proposed \texttt{CFLoRA} framework naturally accommodates clients with heterogeneous computational and memory budgets. Specifically, the server maintains LoRA factors of rank \(R\), while client \(i\) activates only \(r_i\leq R\) channels subject to its resource budget.
Specifically, in round \(t\), the server constructs a selection matrix \(\mathbf{E}_i^t \in \{0,1\}^{R \times r_i}\) for each client $i$, defined as
\begin{equation}
\begin{aligned}
\mathbf E_i^t
=[\mathbf e_{s_{i,1}^t},\ldots,
\mathbf e_{s_{i,r_i}^t}]
\in\{0,1\}^{R\times r_i},
\qquad
(\mathbf E_i^t)^\top\mathbf E_i^t = \mathbf I_{r_i},
\end{aligned}
\label{eq:heterogeneous_selection}
\end{equation}
where \(\mathbf e_j\) is the \(j\)-th standard basis vector of \(\mathbb R^R\). 
The corresponding global activation mask for client \(i\) is then defined as 
\begin{equation}
    \mathbf{H}_i^t = \mathbf{E}_i^t(\mathbf{E}_i^t)^\top \in \{0,1\}^{R \times R},
    \qquad
    \operatorname{rank}(\mathbf{H}_i^t) = r_i.
\end{equation}

Client \(i\) then initializes its LoRA factors by 
extracting the selected channels from the global factors
$\widetilde{\mathbf B}_i^{t,0}
=\mathbf B^t\mathbf E_i^t
\in\mathbb R^{d_{\mathrm{out}}\times r_i}$ and $\widetilde{\mathbf A}_i^{t,0}
=(\mathbf E_i^t)^\top\mathbf A^t
\in\mathbb R^{r_i\times d_{\mathrm{in}}}$.
Their product is therefore $\widetilde{\mathbf B}_i^{t,0}
\widetilde{\mathbf A}_i^{t,0} = \mathbf B^t\mathbf H_i^t\mathbf A^t$, which is the rank-\(r_i\) submodel to initialize local training.
% \paragraph{Local tr0aining and aggregation.}
At local step \(k\), client \(i\) updates according to
\begin{equation}
\begin{aligned}
\widetilde{\mathbf B}_i^{t,k+1}
=
\widetilde{\mathbf B}_i^{t,k}
+\Delta\widetilde{\mathbf B}_i^{t,k}(\mathbf E_i^t)^\top\mathbf Z^t\mathbf E_i^t,
\qquad
\widetilde{\mathbf A}_i^{t,k+1}
=
\widetilde{\mathbf A}_i^{t,k}
+(\mathbf I_{r_i}-(\mathbf E_i^t)^\top\mathbf Z^t\mathbf E_i^t)
\Delta\widetilde{\mathbf A}_i^{t,k},
\end{aligned}
\label{eq:compact_local_update}
\end{equation}
where \(\Delta\widetilde{\mathbf B}_i^{t,k}\) and \(\Delta\widetilde{\mathbf A}_i^{t,k}\) denote client \(i\)'s unmasked local updates at step \(k\).
After \(K\) local steps, the client uploads the updates, which are given by
\begin{equation}
\begin{aligned}
\Delta\widetilde{\mathbf B}_i^t
            = \widetilde{\mathbf B}_i^{t,K}
            -\widetilde{\mathbf B}_i^{t,0} \qquad
\Delta\widetilde{\mathbf A}_i^t
            = \widetilde{\mathbf A}_i^{t,K}
            -\widetilde{\mathbf A}_i^{t,0}.
\end{aligned}
\label{eq:compact_cumulative_updates}
\end{equation}
Then, the server embeds these updates into the \(R\)-rank space and aggregates them by:
\begin{equation}
\mathbf B^{t+1}
=
\mathbf B^t+\sum_{i}\frac{w_i}{q} \Delta\widetilde{\mathbf B}_i^t
(\mathbf E_i^t)^\top,
\qquad
\mathbf A^{t+1}
=
\mathbf A^t+\sum_{i}\frac{w_i}{q} \mathbf E_i^t
\Delta\widetilde{\mathbf A}_i^t.
\label{eq:heterogeneous_aggregation}
\end{equation}
We establish the following proposition to show the exactness of aggregation under heterogeneous client ranks.
\begin{proposition}[Error-Free Aggregation of \texttt{CFLoRA} with Heterogeneous Client Ranks]
\label{prop:heterogeneous_exact}
Suppose all participating clients use the common complementary mask \(\mathbf Z^t\), while their ranks \(r_i\) and selection matrices \(\mathbf E_i^t\) may differ. The factor-wise aggregation rule in \eqref{eq:heterogeneous_aggregation} satisfies
\begin{equation}
\begin{aligned}
\left(\mathbf B^t + \sum_{i} \frac{w_i}{q} \Delta\widetilde{\mathbf B}_i^t (\mathbf E_i^t)^\top\right) 
\left(\mathbf A^t + \sum_{i} \frac{w_i}{q} \mathbf E_i^t \Delta\widetilde{\mathbf A}_i^t\right) = 
\mathbf B^t\mathbf A^t + \sum_{i} \frac{w_i}{q} \left( \widetilde{\mathbf B}_i^{t,K} \widetilde{\mathbf A}_i^{t,K} - \mathbf B^t\mathbf H_i^t\mathbf A^t \right).
\end{aligned}
\label{eq:heterogeneous_exactness}
\end{equation}
\end{proposition}

Proposition~\ref{prop:heterogeneous_exact} shows that \texttt{CFLoRA} still achieves exact federated LoRA aggregation under heterogeneous client ranks. The proof can be found in Appendix~\ref{app:proof_hete_exact_agg}.

\subsection{Complementary-channel Scaling}
\label{sec:channel_selection} 

To prevent biased updates arising from channel selection, the protocol introduces a unified scaling compensation factor. 
The numbers of trainable channels assigned to \(\mathbf B\) and \(\mathbf A\) are
\begin{equation}
    k_{B,i}^t=\operatorname{tr}(\widetilde{\mathbf Z}_i^t),
\qquad
k_{A,i}^t=r_i-k_{B,i}^t,
\end{equation}
where $\widetilde{\mathbf Z}_i^t=(\mathbf E_i^t)^\top\mathbf Z^t\mathbf E_i^t$, and we have $\mathbb E[k_{B,i}^t]=r_i p$ and $\mathbb E[k_{A,i}^t]=r_i(1-p)$.
Then, to compensate for the unequal probabilities of selecting the two factors, we scale the gradients by \(1/p\) and \(1/(1-p)\), respectively.
Specifically, for \(0<p<1\), the factors are updated as
\begin{equation}
\begin{aligned}
\widetilde{\mathbf B}_i^{t,k+1}
&=
\widetilde{\mathbf B}_i^{t,k}
-
\frac{\eta}{p}\nabla_{\widetilde{\mathbf B}}\ell\!\left(\mathbf W_0+
\widetilde{\mathbf B}_i^{t,k}\widetilde{\mathbf A}_i^{t,k};\xi\right)\widetilde{\mathbf Z}_i^t,
\\
\widetilde{\mathbf A}_i^{t,k+1}
&=
\widetilde{\mathbf A}_i^{t,k}
-\frac{\eta}{1-p}
(\mathbf I_{r_i}-\widetilde{\mathbf Z}_i^t)
\nabla_{\widetilde{\mathbf A}}\ell\!\left(
\mathbf W_0+\widetilde{\mathbf B}_i^{t,k}\widetilde{\mathbf A}_i^{t,k};\xi\right),
\end{aligned}    
\end{equation}
where \(\eta\) is the learning rate. Note that setting $r_i = R$, which implies $\mathbf{E}_i^t = \mathbf{I}_R$ and $\widetilde{\mathbf{Z}}_i^t = \mathbf{Z}^t$, recovers the homogeneous update rule. The activation matrix \(\mathbf E_i^t\) and complementary mask \(\widetilde{\mathbf Z}_i^t\) remain fixed throughout the round. 
% Each reciprocal-probability correction is applied exactly once per local step. 
Since the updates defined in Eq. \ref{eq: homogeneous_updates} and Eq. \ref{eq:compact_cumulative_updates} can naturally absorb this scaling, Propositions \ref{prop:exact} and \ref{prop:heterogeneous_exact} continue to guarantee exact aggregation.
% The embedded updates retain the global complementary support and consequently, complementarity is preserved even when clients activate different numbers and subsets of channels. Therefore, error-free aggregation remains under heterogeneous ranks, as established in Proposition \ref{prop:heterogeneous_exact}.

Considering practical partial participation (selecting $m$ out of $N$ clients per round), the procedure of \texttt{CFLoRA} is summarized in Algorithm \ref{alg:CFLoRA-homogeneous} for homogeneous cases and Algorithm \ref{alg:CFLoRA} for heterogeneous cases.

% ------------------ Convergence Analysis ------------------
\subsection{Convergence Analysis}
\label{sec:convergence} 
Next, we theoretically analyze the convergence of \texttt{CFLoRA} under partial client participation. Our analysis naturally encompasses full participation as a special case. 

\textbf{Notations.} For each client $i$, let $\mathcal H_i$ denote the uniform distribution over diagonal binary masks with  $r_i$ active channels, satisfying
$\mathbb E_{\mathbf H\sim\mathcal H_i}[\mathbf H]=\frac{r_i}{R}\mathbf I_R$. The activation mask
$\mathbf H_i^t=\mathbf E_i^t(\mathbf E_i^t)^\top\sim\mathcal H_i$
is sampled independently across clients and independently of $\mathbf Z^t$.
Let $\mathbf X=(\mathbf B,\mathbf A)$, for a given activation mask $\mathbf H$, the mask-conditioned and expected local objectives for client $i$ are respectively defined as
\begin{equation}
f_i^{\mathbf H}(\mathbf X)
=
\mathbb E_{\xi\sim\mathcal D_i}
\bigl[\ell(\mathbf W_0+\mathbf B\mathbf H\mathbf A;\xi)\bigr],
\qquad
\phi_i(\mathbf X)
=
\mathbb E_{\mathbf H\sim\mathcal H_i}
\bigl[f_i^{\mathbf H}(\mathbf X)\bigr],
\end{equation}
and therefore the global objective is formulated as
\begin{equation}
   \Phi(\mathbf X)=\sum_{i=1}^{N}w_i\phi_i(\mathbf X). 
\end{equation}
Notably, when $r_i=R$ for all clients,
$\mathbf H_i=\mathbf I_R$ and $\Phi(\mathbf X)$ reduces to the standard objective $F(\mathbf X)$.

We assume that each client's local objective is $L$-smooth and the global objective is bounded (Assumption~\ref{assump:smooth}), local stochastic gradients are unbiased with bounded variance (Assumption~\ref{assump:unbias}), which are detailed in Appendix \ref{app:assumptions}. With notations above and Lemma~\ref{lem:compact_unbiased} on unbiased estimators of the scaled aggregated gradients (detailed in Appendix~\ref{app:proof_lemma}), we establish the following convergence analysis.

% the scaled aggregated gradients are unbiased estimators of the global gradients (Lemma~\ref{lem:compact_unbiased}). Formal mathematical statements and corresponding proofs are deferred to Appendix~\ref{app:assumptions} and Appendix~\ref{app:proof_lemma},

% Assumption~\ref{assump:smooth} and 
% Assumption~\ref{assump:unbias} (in Appendix \ref{app:assumptions}) and Lemma \ref{lem:compact_unbiased} (in \ref{app:proof_lemma}), we establish the following convergence analysis.

%Our analysis builds upon standard regularity conditions widely used in non-convex federated optimization \citep{bottou2018optimization, fang2026fslora}. Specifically, we assume that each client's local objective is $L$-smooth and the global objective is bounded below (Assumption~\ref{assump:smooth}); local stochastic gradients are unbiased with bounded variances for both minibatch and mask sampling, alongside a uniformly bounded second moment (Assumption~\ref{assump:unbias}); and the scaled aggregated gradients are unbiased estimators of the global gradients (Lemma~\ref{lem:compact_unbiased}). Formal mathematical statements and corresponding proofs are deferred to Appendix~\ref{app:assumptions} and Appendix~\ref{app:proof_lemma}, which together establish our convergence guarantees.

% To quantify client-sampling variability, define
% \begin{equation}
% \label{eq:client_sampling_dispersion}

% \end{equation}
% Assumption~\ref{assump:unbias} implies
% $\zeta_w^2\le NM\sum_iw_i^2$,
% so no additional bounded-heterogeneity assumption is required.

\begin{theorem}[Convergence of \texttt{CFLoRA} with Heterogeneous Client Ranks]
\label{thm:convergence}
Suppose that Assumptions~\ref{assump:smooth}--\ref{assump:unbias} hold and $\Delta_0>0$. With the learning rate
\begin{equation}
\label{eq:explicit_learning_rate}
\eta_T
=
\frac{1}{K}
\min\left\{
\frac{1}{4L\kappa},
\sqrt{\frac{\Delta_0}{L\mathcal V_K T}},
\left(\frac{2\Delta_0}{\mathcal D_K T}\right)^{1/3}
\right\},
\end{equation}
\texttt{CFLoRA} satisfies
\begin{equation}
\label{eq:convergence_bound}
\frac{1}{T}\sum_{t=0}^{T-1}
\mathbb E\left[
\|\nabla\Phi(\mathbf X^t)\|^2
\right]
\le
4\sqrt{\frac{L\Delta_0\mathcal V_K}{T}}
+
2\left(\frac{4\Delta_0^2\mathcal D_K}{T^2}\right)^{1/3}
+
\frac{8L\kappa\Delta_0}{T},
\end{equation}
where
$\mathcal V_K
=
\kappa\left(
\frac{N\tau_w^2}{m}
+
\frac{N-m}{m(N-1)}\zeta_w^2
\right)
+
\frac{\kappa^2\sigma_w^2}{qK},
\qquad
\mathcal D_K
=
\frac{4+\chi_m}{8}
L^2\kappa^3M
\left(1-\frac{1}{K}\right)^2,
$
$
\chi_m
=
1+\frac{N-m}{m(N-1)}\left(
N\sum_{i}w_i^2-1
\right),
$
$\sigma_w^2=\sum_{i=1}^{N}w_i^2\sigma_i^2$,
$\tau_w^2=\sum_{i=1}^{N}w_i^2\tau_i^2$,
$\kappa=p_{\min}^{-1}$, $p_{\min}=\min\{p, 1-p\}$,
$\Delta_0=\mathbb E[\Phi(\mathbf X^0)]-\Phi^\star$, $\zeta_w^2
=
\sup_{\mathbf X\in\mathcal C}
\left\{
N\sum_{i}w_i^2
\|\nabla\phi_i(\mathbf X)\|^2
-
\|\nabla\Phi(\mathbf X)\|^2
\right\}$, and $\mathbf X^0$ denotes the initial iterate.
For $N=1$, set $\frac{N-m}{m(N-1)}=0$.
\end{theorem}

The proof is provided in Appendix \ref{app:proof_convergence}. 
Due to the exact aggregation, the convergence bound in Theorem~\ref{thm:convergence} achieves an $\mathcal O(1/\sqrt T)$ convergence rate without non-vanishing bilinear aggregation-error terms. Heterogeneous client ranks affect both the objective $\Phi$ and the activation-mask variance $\tau_w^2$, which vanishes when $r_i=R$ for every client. Partial participation introduces the client-sampling term $\frac{N-m}{m(N-1)}\zeta_w^2$. More discussions are provided in Appendix~\ref{app:insights_convergence}.

\begin{corollary}[Convergence of \texttt{CFLoRA} with Homogeneous Client Ranks]
\label{cor:homogeneous}
Under the conditions of Theorem~\ref{thm:convergence},
suppose $r_i=R$ for every client.
Then $\Phi=F$ and $\tau_w^2=0$.
Choosing
$\eta_T=(4L\kappa K\sqrt T)^{-1}$,
the server iterates satisfy
\begin{equation}
\label{eq:homogeneous_rate}
\frac{1}{T}\sum_{t=0}^{T-1}
\mathbb E\!\left[
\|\nabla F(\mathbf X^t)\|^2
\right]
\le
\frac{
8L\kappa\Delta_0
+\frac{N-m}{m(N-1)}\zeta_w^2/2
+\kappa\sigma_w^2/(2qK)
}{\sqrt T}
+
\frac{(4+\chi_m)\kappa M}{128T}
\left(1-\frac{1}{K}\right)^2.
\end{equation}
\end{corollary}

The proof is provided in Appendix \ref{app:proof_homogeneous}. 

\paragraph{Significance of Corollary~\ref{cor:homogeneous}.} Under homogeneous ranks, there is no objective discrepancy induced by lower client ranks. As a result, the $\mathcal O(1/\sqrt T)$ convergence rate applies directly to the original LoRA objective $F$ in \eqref{eq:objective}. This results from the unique advantage of \texttt{CFLoRA}: it co-adapts both factors while entirely avoiding the aggregation errors in standard federated LoRA.

%In Corollary~\ref{cor:homogeneous}, 

\paragraph{Determination of channel sampling probability $p$.}
The parameter $\kappa=p_{\min}^{-1}$ reveals that low selection probabilities for either factor amplify the upper bound, underscoring the need to update both factors simultaneously. Since the upper bounds in Theorem~\ref{thm:convergence} and Corollary~\ref{cor:homogeneous} decrease as $p_{\min}=\min\{p,1-p\}$ increases, they are minimized at $p=0.5$. However, this stems from a worst-case theoretical bound, which may not be optimal empirically.

We conduct extensive empirical studies on $p$ (Tables~\ref{table:ablation_robert_homo_10of25}-\ref{table:ablation_robert_hete_10of25} and  Tables~\ref{table:ablation_llama_homo_10of25}-\ref{table:ablation_llama_hete_10of25}), yielding two intriguing insights. On the one hand, increasing $p$ from 0.3 to 0.9 generally improves performance. This aligns with prior findings \citep{hayou2024lora+} that the two low-rank matrices contribute asymmetrically to the optimization process, and prioritizing the $\mathbf{B}$ matrix update is beneficial. On the other hand, setting $p=1.0$ degrades performance, which is consistent with our theoretical analysis that the co-adaptation of both factors remains crucial \footnote{Note that for $p=1.0$, the scaling step described in Section~\ref{sec:channel_selection} is omitted.}. Consequently, we adopt $p=0.9$ as the default setting across all experiments.

%2) setting $p=0.5$ to $p=0.9$ shows robust performance, confirming our robustness to parameter sensitivity, 

% confirms that $p=0.5$ achieves good average empirical performance across six datasets. 

%The distinction also affects the benefit of local computation: at a fixed effective round step size $\gamma=\eta K$, the remaining additive sampling-noise term scales as $\kappa^2\sigma_w^2/K$, while heterogeneous-rank training additionally incurs the activation-mask term $\kappa\tau_w^2$. Complementary selection remains stochastic in the homogeneous case, as reflected by $\kappa$. Thus, both settings retain an $\mathcal O(T^{-1/2})$ stationarity guarantee, while homogeneous full-rank training removes a source of sampling variability and recovers the full-adapter objective.

\subsection{Resource Efficiency of \texttt{CFLoRA}} \label{subsec:efficiency}
Compared to standard federated LoRA at the same active rank, complementary updates provide efficiency gains across several key aspects: i) \textbf{Trainable Parameters}: By freezing the rows and/or columns, \texttt{CFLoRA} can reduce the number of trainable parameters by $50\%$ per round \footnote{The $50\%$ reduction refers to square attention projection matrices, so each trainable row of $\mathbf{A}$ and column of $\mathbf{B}$ contains the same number of parameters. For rectangular matrices, the reduction depends on their dimensions and the selected channels.}. ii) \textbf{Communication Cost}: By transmitting only the trainable parameters, \texttt{CFLoRA} reduces the adapter uploading cost by $50\%$, though downlink of full factors is still needed for clients to perform forward passes. A comprehensive analysis is provided in Appendix~\ref{app:implementation_complexity}.

% ------------------------ Experiments ------------------------

% ---------------------- Table GLUE Homo 10-of-25 -----------------------
\begin{table*}[!t]
\centering
% \resizebox{\textwidth}{!}{
% \setlength{\tabcolsep}{3mm}  % 保持原压缩
\caption{Test accuracy for fine-tuning the RoBERTa model on the GLUE benchmark under \textit{homogeneous} client ranks.}
\label{table:main_results_robert_homo_10of25} 
\resizebox{\textwidth}{!}{
\renewcommand{\arraystretch}{1}
\begin{tabular}{l|cccccc|c}
\toprule
\textbf{Method} & \textbf{QNLI} & \textbf{MRPC} & \textbf{CoLA} & \textbf{MNLI-M} & \textbf{MNLI-MM} & \textbf{QQP} & \textbf{Avg.} \\
\midrule
FedIT    & 86.61 {\scriptsize \color{gray} $\pm$ 0.26} & 72.27 {\scriptsize \color{gray} $\pm$ 2.32} & 70.49 {\scriptsize \color{gray} $\pm$ 0.90} & 79.10 {\scriptsize \color{gray} $\pm$ 2.21} & \textbf{79.89} {\scriptsize \color{gray} $\pm$ 2.28} & 71.75 {\scriptsize \color{gray} $\pm$ 1.35} & 76.69 \\
FFA-LoRA & 88.49 {\scriptsize\color{gray}$\pm$ 0.23}
& 70.58 {\scriptsize\color{gray}$\pm$ 1.56}
& 73.81 {\scriptsize\color{gray}$\pm$ 1.74}
& 72.42 {\scriptsize\color{gray}$\pm$ 3.27}
& 73.39 {\scriptsize\color{gray}$\pm$ 3.08}
& 73.26 {\scriptsize\color{gray}$\pm$ 1.38} & 75.33\\
RoLoRA   & 87.60 {\scriptsize \color{gray} $\pm$ 2.21} & 73.16 {\scriptsize \color{gray} $\pm$ 2.11} & 74.40 {\scriptsize \color{gray} $\pm$ 0.89} & 72.01 {\scriptsize \color{gray} $\pm$ 1.55} & 72.50 {\scriptsize \color{gray} $\pm$ 3.56} & 73.47 {\scriptsize \color{gray} $\pm$ 0.41} & 75.52 \\
AS-LoRA  & 88.26 {\scriptsize \color{gray} $\pm$ 0.67} & 72.30 {\scriptsize \color{gray} $\pm$ 1.06} & 73.57 {\scriptsize \color{gray} $\pm$ 2.15} & 71.41 {\scriptsize \color{gray} $\pm$ 1.88} & 72.60 {\scriptsize \color{gray} $\pm$ 1.51} & 75.23 {\scriptsize \color{gray} $\pm$ 2.57} & 75.56\\
FLoRA    & 89.27 {\scriptsize \color{gray} $\pm$ 0.20} & 71.46 {\scriptsize \color{gray} $\pm$ 1.24} & 76.81 {\scriptsize \color{gray} $\pm$ 2.74} & 70.56 {\scriptsize \color{gray} $\pm$ 0.82} & 71.52 {\scriptsize \color{gray} $\pm$ 0.21} & \textbf{76.43} {\scriptsize \color{gray} $\pm$ 2.74} & 76.01\\
\midrule
\textbf{CFLoRA} & \textbf{90.61} {\scriptsize\color{gray}$\pm$ 0.25} & \textbf{73.53} {\scriptsize\color{gray}$\pm$ 1.59} & \textbf{78.95} {\scriptsize\color{gray}$\pm$ 1.48} & \textbf{79.16} {\scriptsize\color{gray}$\pm$ 1.22} & 79.85 {\scriptsize\color{gray}$\pm$ 1.01} & 76.30 {\scriptsize\color{gray}$\pm$ 0.36} & \textbf{79.73}\\
\bottomrule
\end{tabular}
}
\end{table*}

% ------------------------ Table GLUE Hete 10-of-25 ----------------------
\begin{table*}[!t]
\centering
% \resizebox{\textwidth}{!}{
% \setlength{\tabcolsep}{3mm}  % 保持原压缩
\caption{Testing accuracy for fine-tuning the RoBERTa model on the GLUE benchmark under \textit{heterogeneous} client ranks.}
\label{table:main_results_robert_hete_10of25}
\resizebox{\textwidth}{!}{
\renewcommand{\arraystretch}{1}
\begin{tabular}{l|cccccc|c}
\toprule
\textbf{Method} & \textbf{QNLI} & \textbf{MRPC} & \textbf{CoLA} & \textbf{MNLI-M} & \textbf{MNLI-MM} & \textbf{QQP} & \textbf{Avg.} \\
\midrule
FedIT    & 70.34 {\scriptsize \color{gray} $\pm$ 2.01} & 69.60 {\scriptsize \color{gray} $\pm$ 2.43} & 73.56 {\scriptsize \color{gray} $\pm$ 3.67} & 58.34 {\scriptsize \color{gray} $\pm$ 1.03} & 58.22 {\scriptsize \color{gray} $\pm$ 0.93} &  73.12 {\scriptsize \color{gray} $\pm$ 1.41} & 67.20
\\
FFA-LoRA & 86.86 {\scriptsize \color{gray} $\pm$ 2.15} & 70.67 {\scriptsize \color{gray} $\pm$ 2.31} & 73.09 {\scriptsize \color{gray} $\pm$ 2.97} & 63.74 {\scriptsize \color{gray} $\pm$ 5.02} & 64.04 {\scriptsize \color{gray} $\pm$ 4.88} & 75.91 {\scriptsize \color{gray} $\pm$ 1.01}& 72.39
\\
RoLoRA   & 84.87 {\scriptsize \color{gray} $\pm$ 3.10} & 68.46 {\scriptsize \color{gray} $\pm$ 3.12} & 73.14 {\scriptsize \color{gray} $\pm$ 3.40} & 51.19 {\scriptsize \color{gray} $\pm$ 3.95} & 51.60 {\scriptsize \color{gray} $\pm$ 2.84} & 74.73 {\scriptsize \color{gray} $\pm$ 1.29}& 
67.33
\\
AS-LoRA  & 80.71 {\scriptsize \color{gray} $\pm$ 2.52} & 71.12 {\scriptsize \color{gray} $\pm$ 1.33} & 73.09 {\scriptsize \color{gray} $\pm$ 2.97} & 63.60 {\scriptsize \color{gray} $\pm$ 4.86} & 63.95 {\scriptsize \color{gray} $\pm$ 5.94} & 70.73 {\scriptsize \color{gray} $\pm$ 1.94} &
70.53
\\
FLoRA     & 87.35 {\scriptsize \color{gray} $\pm$ 1.84} & 69.80 {\scriptsize \color{gray} $\pm$ 1.71} & 75.75 {\scriptsize \color{gray} $\pm$ 1.34} & 72.22 {\scriptsize \color{gray} $\pm$ 1.26} & 70.97 {\scriptsize \color{gray} $\pm$ 1.14}  & 76.61 {\scriptsize \color{gray} $\pm$ 0.53} & 75.45
\\
FlexLoRA  & 86.66 {\scriptsize \color{gray} $\pm$ 1.55} & 71.92 {\scriptsize \color{gray} $\pm$ 1.19} &  77.60 {\scriptsize \color{gray} $\pm$ 0.58} & 71.35 {\scriptsize \color{gray} $\pm$ 1.49} & 72.25 {\scriptsize \color{gray} $\pm$ 1.48} & 76.40 {\scriptsize \color{gray} $\pm$ 1.32}
& 76.03 
\\
FSLoRA    & 86.68 {\scriptsize \color{gray} $\pm$ 2.27} & \textbf{75.95} {\scriptsize \color{gray} $\pm$ 1.34} & 74.95 {\scriptsize \color{gray} $\pm$ 0.52} & 72.55 {\scriptsize \color{gray} $\pm$ 1.57} & 73.49 {\scriptsize \color{gray} $\pm$ 1.62} & \textbf{79.58} {\scriptsize \color{gray} $\pm$ 1.29} & 77.20 
\\
\midrule
\textbf{CFLoRA} & \textbf{91.29} {\scriptsize\color{gray}$\pm$ 0.45} & 74.86 {\scriptsize\color{gray}$\pm$ 0.69} & \textbf{77.80} {\scriptsize\color{gray}$\pm$ 1.26} & \textbf{76.82} {\scriptsize\color{gray}$\pm$ 0.80} & \textbf{77.32} {\scriptsize\color{gray}$\pm$ 1.14} & 79.34 {\scriptsize \color{gray} $\pm$ 0.87} & \textbf{79.57}
\\
\bottomrule
\end{tabular}
}
\end{table*}

\section{Experiments}
\label{sec:experiments}

\paragraph{Implementation Details and Baselines.}
Our evaluation focuses on RoBERTa (125M) \citep{liu2019roberta} on the GLUE benchmark \citep{wang2018glue} and LLaMA-3.2-3B-Instruct \citep{grattafiori2024llama} on commonsense reasoning tasks \citep{hu2023llm}. Following prior work \citep{zhang2024federatedgpt, fang2026fslora}, we adopt Dirichlet-based partitioning for non-IID dataset splits. 
In our default evaluation, we consider a pool of 25 clients and randomly select 10 per round for training. We also present an extended analysis by varying the number of clients in Appendix~\ref{app:full_participation}. All results are averaged over 4 independent seeds. Further experimental details are provided in Appendix~\ref{app:exp_details}. 

We compare \texttt{CFLoRA} against several state-of-the-art federated LoRA baselines, including FedIT \citep{zhang2024federatedgpt}, the standard factor-wise averaging baseline; FFA-LoRA \citep{sun2024ffa}, RoLoRA \citep{chen2025rolora}, FLoRA \citep{wang2024flora}, and AS-LoRA \citep{kim2026adaptive}, which are designed to provide error-free aggregation; and FlexLoRA \citep{bai2024flexlora} and FSLoRA \citep{fang2026fslora}, which are explicitly designed to accommodate heterogeneous client ranks.

% \subsection{Main Results}

\subsection{Main Results}

% \textbf{Comparison under Homogeneous Client Ranks.}
% Table~\ref{table:main_results_llama_homo_10of25} summarizes the
% results for fine-tuning LLaMA-3.2-3B-Instruct on commonsense reasoning
% tasks under homogeneous client ranks. \texttt{CFLoRA} achieves the
% highest average accuracy of $74.28\%$, outperforming the strongest
% baseline, FLoRA, by $1.45$ percentage points and standard federated
% LoRA (FedIT) by $3.25$ points. 

% \textbf{Comparison under Heterogeneous Client Ranks.}
% Table~\ref{table:main_results_llma_reasoning_hete_10of25} reports
% the corresponding results under heterogeneous client ranks.
% \texttt{CFLoRA} achieves the highest average accuracy of $75.36\%$,
% exceeding FSLoRA, the strongest baseline, by $2.16$ percentage points.
% It also outperforms FLoRA and FlexLoRA by $2.76$ and $3.77$ points,
% respectively. \texttt{CFLoRA} obtains the highest scores on four
% of the six tasks, with particularly clear gains on PIQA and SIQA,
% where it exceeds the best baseline scores by $3.77$ and $4.37$
% points, respectively. These results demonstrate the effectiveness
% of complementary updates when clients train compact adapters
% with different ranks.

\textbf{Comparison under Homogeneous Client Ranks.}
Table~\ref{table:main_results_robert_homo_10of25} summarizes the
results for fine-tuning RoBERTa on the GLUE benchmark under
homogeneous client ranks. \texttt{CFLoRA} achieves the highest
average accuracy of $79.73\%$, outperforming the strongest baseline,
FedIT, by $3.04$ percentage points. It obtains the highest scores
on QNLI, MRPC, CoLA, and MNLI-M, with improvements of $1.34$ and
$2.14$ points over the best baseline scores on QNLI and CoLA,
respectively. It also remains competitive on MNLI-MM and QQP.
These results demonstrate that complementary updates enable effective adaptation despite training fewer adapter parameters per round. 

\textbf{Comparison under Heterogeneous Client Ranks.}
Table~\ref{table:main_results_robert_hete_10of25} reports the
corresponding results under heterogeneous client ranks.
\texttt{CFLoRA} achieves the highest average accuracy of $79.57\%$,
exceeding FSLoRA, the strongest baseline, by $2.37$ percentage
points. It achieves the highest scores
on QNLI, CoLA, MNLI-M, and MNLI-MM, with gains of $3.94$, $4.27$,
and $3.83$ points over the best baseline scores on QNLI, MNLI-M,
and MNLI-MM, respectively. These results support the effectiveness
of complementary updates when clients train compact adapters
with different rank budgets. Additional results for
LLaMA-3.2-3B-Instruct under homogeneous and heterogeneous ranks
are reported in Tables~\ref{table:main_results_llama_homo_10of25}
and~\ref{table:main_results_llma_reasoning_hete_10of25},
respectively, and discussed in Appendix~\ref{app:llama_reasoning_hete}.

\textbf{Parameter and Communication Efficiency.}
Importantly, these accuracy improvements of \texttt{CFLoRA} are accompanied by reduced resource costs. Figure~\ref{fig:Comm} (in Appendix~\ref{app:implementation_complexity}) illustrates the trade-off between accuracy and communication cost over $200$ rounds. On MNLI, \texttt{CFLoRA} reaches approximately $70\%$ matched accuracy with only about $8$~GB of cumulative communication, whereas FSLoRA requires roughly $14$~GB to attain comparable performance. On QNLI, \texttt{CFLoRA} exceeds $90\%$ accuracy at approximately $8$~GB, also making it significantly more communication-efficient than the baselines. In comparison, FLoRA is not communication-efficient because its effective rank is the sum of all client ranks. Overall, these results demonstrate that \texttt{CFLoRA} achieves strong accuracy within smaller communication budget. A detailed cost analysis is provided in Appendix~\ref{app:implementation_complexity}.

\begin{table}[t]
\centering
\caption{Ablation study on selection probability $p$ for RoBERTa on GLUE tasks under homogeneous client ranks.}
\label{table:ablation_robert_homo_10of25}
\resizebox{\textwidth}{!}{
\renewcommand{\arraystretch}{1}
\begin{tabular}{l|cccccc|c}
\toprule
\textbf{Schedule} & \textbf{QNLI} & \textbf{MRPC} & \textbf{CoLA}
& \textbf{MNLI-M} & \textbf{MNLI-MM} & \textbf{QQP} & \textbf{Avg.} \\
\midrule
$p=0.3$ & 87.80 
{\scriptsize\color{gray}$\pm$ 0.02} & 70.83 {\scriptsize\color{gray}$\pm$ 1.73} & 73.33 {\scriptsize\color{gray}$\pm$ 2.37} & 71.15 {\scriptsize\color{gray}$\pm$ 1.07} & 72.73 {\scriptsize\color{gray}$\pm$ 1.03} & 72.45 {\scriptsize\color{gray}$\pm$ 1.39} & 74.72 \\
$p=0.5$ & 90.64 {\scriptsize\color{gray}$\pm$ 0.25} & 72.43 {\scriptsize\color{gray}$\pm$ 2.50} & 78.72 {\scriptsize\color{gray}$\pm$ 1.73} & 77.82 {\scriptsize\color{gray}$\pm$ 0.49} & 78.00 {\scriptsize\color{gray}$\pm$ 0.12} & 74.57 {\scriptsize\color{gray}$\pm$ 0.43} & 78.70
\\
$p=0.7$ & \textbf{90.69} {\scriptsize\color{gray}$\pm$ 0.24} & 72.55 {\scriptsize\color{gray}$\pm$ 2.06} & \textbf{79.81} {\scriptsize\color{gray}$\pm$ 1.52} & 77.47 {\scriptsize\color{gray}$\pm$ 0.61} & 78.02 {\scriptsize\color{gray}$\pm$ 0.65} & 74.52 {\scriptsize\color{gray}$\pm$ 1.43} & 78.84
\\
$p=0.9$ & 90.61 {\scriptsize\color{gray}$\pm$ 0.25} & \textbf{73.53} {\scriptsize\color{gray}$\pm$ 1.59} & 78.95 {\scriptsize\color{gray}$\pm$ 1.48} & \textbf{79.16} {\scriptsize\color{gray}$\pm$ 1.22} & \textbf{79.85} {\scriptsize\color{gray}$\pm$ 1.01} & \textbf{76.30} {\scriptsize\color{gray}$\pm$ 0.36} & \textbf{79.73}
\\
$p=1.0$ & 88.61 {\scriptsize\color{gray}$\pm$ 0.27}
& 70.47 {\scriptsize\color{gray}$\pm$ 1.54}
& 73.89 {\scriptsize\color{gray}$\pm$ 1.77}
& 72.36 {\scriptsize\color{gray}$\pm$ 3.23}
& 73.54 {\scriptsize\color{gray}$\pm$ 3.13}
& 73.17 {\scriptsize\color{gray}$\pm$ 1.39}
& 75.34
\\
\bottomrule
\end{tabular}
}
\end{table}

% ----------------------- p on GLUE tasks hete --------------------
\begin{table}[t]
\centering
\caption{Ablation study on selection probability $p$ for RoBERTa on GLUE tasks under heterogeneous client ranks.}
\label{table:ablation_robert_hete_10of25}
\resizebox{\textwidth}{!}{
\renewcommand{\arraystretch}{1}
\begin{tabular}{l|cccccc|c}
\toprule
\textbf{Schedule} & \textbf{QNLI} & \textbf{MRPC} & \textbf{CoLA} & \textbf{MNLI-M} & \textbf{MNLI-MM} & \textbf{QQP} & \textbf{Avg.} \\
\midrule
$p=0.3$ & 88.31 {\scriptsize\color{gray}$\pm$ 0.99} & 68.79 {\scriptsize\color{gray}$\pm$ 0.35} & 71.46 {\scriptsize\color{gray}$\pm$ 1.27} & 68.78 {\scriptsize\color{gray}$\pm$ 0.49} & 69.66 {\scriptsize\color{gray}$\pm$ 1.29} & 73.21 {\scriptsize\color{gray}$\pm$ 1.15} & 73.37\\
$p=0.5$ & 91.16 {\scriptsize\color{gray}$\pm$ 0.10} & 73.98 {\scriptsize\color{gray}$\pm$ 2.26} & 76.90 {\scriptsize\color{gray}$\pm$ 2.38} & 76.37 {\scriptsize\color{gray}$\pm$ 1.02} & 76.85 {\scriptsize\color{gray}$\pm$ 1.30} & 77.16 {\scriptsize\color{gray}$\pm$ 1.28} & 78.74
\\
$p=0.7$ & \textbf{91.41} {\scriptsize\color{gray}$\pm$ 0.87} & 74.13 {\scriptsize\color{gray}$\pm$ 1.04} & 77.32 {\scriptsize\color{gray}$\pm$ 0.85} & 76.74 {\scriptsize\color{gray}$\pm$ 1.37} & \textbf{77.51} {\scriptsize\color{gray}$\pm$ 0.73} & \textbf{79.61} {\scriptsize\color{gray}$\pm$ 1.07}
& 79.45 \\
$p=0.9$ & 91.29 {\scriptsize\color{gray}$\pm$ 0.45} & \textbf{74.86} {\scriptsize\color{gray}$\pm$ 0.69} & \textbf{77.80} {\scriptsize\color{gray}$\pm$ 1.26} & \textbf{76.82} {\scriptsize\color{gray}$\pm$ 0.80} & 77.32 {\scriptsize\color{gray}$\pm$ 1.14}  & 79.34 {\scriptsize \color{gray} $\pm$ 0.87} & \textbf{79.57}
\\
$p=1.0$ & 89.90 {\scriptsize\color{gray}$\pm$ 1.39} & 73.00 {\scriptsize\color{gray}$\pm$ 2.02} & 74.51 {\scriptsize\color{gray}$\pm$ 1.42} & 75.25 {\scriptsize\color{gray}$\pm$ 1.57} & 75.81 {\scriptsize\color{gray}$\pm$ 0.91} & 73.63 {\scriptsize\color{gray}$\pm$ 1.95} & 77.02 \\
\bottomrule
\end{tabular}
}
\end{table}

% \subsection{Parameter and Communication Efficiency}
% \label{subsec:exp_efficiency}
% Moreover, by updating only one factor component per active channel, \texttt{CFLoRA} halves the number of trainable parameters per round compared to full-factor methods. This maintains competitive adaptation capacity with computational and communication savings.

\subsection{Ablation Study: Impact of channel sampling probability $p$}
\label{subsec:ablations}

In Table \ref{table:ablation_robert_homo_10of25}-Table \ref{table:ablation_robert_hete_10of25} and Table 
\ref{table:ablation_llama_homo_10of25} and Table \ref{table:ablation_llama_hete_10of25} (in Appendix \ref{app:ablation_study}), we conduct experiments with different $p$ on various models in both homogeneous and heterogeneous settings, which reveal a similar phenomenon. In addition to the theoretical finding discussed in Section \ref{sec:convergence} that low selection probabilities for either factor amplify the convergence bound, we empirically observe that \texttt{CFLoRA} remains good for different values of $p > 0.5$ (excluding the boundary case of $p=1.0$). While $p=0.9$ is generally the best, the performance robustness across a wide range of $p$ demonstrates \texttt{CFLoRA}'s general insensitivity to this hyperparameter, making it easy to configure in practice. More discussions are provided in Section \ref{app:ablation_study}.

% In Table~\ref{table:ablation_robert_homo_10of25}, Table\ref{table:ablation_robert_hete_10of25}, and Table\ref{table:ablation_llama_homo_10of25}, we compare adaptive allocation by varying $p$ from $0.3$ to $1$, where $p=1$ indeed match $\mathbf B$-only training, i.e., FFA-LoRA. As aforementioned, ....

% This validates our core motivation: simultaneously co-adapting selected components of both $\mathbf A$ and $\mathbf B$ is critical for maximizing federated adaptation performance.

\section{Conclusion}
We presented CFLoRA, a federated LoRA framework that can update both factors in each round while preserving exact factor-wise aggregation. Our key idea is to utilize a shared complementary mask to assign each latent channel to either the $\mathbf{A}$ or $\mathbf{B}$ factor, eliminating the bilinear terms that cause the aggregation mismatch. Furthermore, \texttt{CFLoRA} seamlessly extends to settings with heterogeneous client ranks by applying additional activation masks. Due to the exact factor-wise aggregation, our analysis establishes an $\mathcal{O}(1/\sqrt{T})$ stationarity rate for the original LoRA objective with homogeneous ranks and for a surrogate objective with heterogeneous ranks. Experiments with RoBERTa and LLaMA-3.2-3B-Instruct show strong performance while reducing trainable parameters and uplink communications.

\textbf{Limitations.} While our complementary-channel design preserves an $\mathcal{O}(1/\sqrt{T})$ convergence rate for the original LoRA objective under homogeneous ranks, the result only applies to a surrogate objective in heterogeneous scenarios. This theoretical gap remains open. Additionally, there is a discrepancy between the worst-case convergence bound, which is minimized at channel selection probability $p=0.5$, and our empirical findings, which favor an asymmetric allocation (e.g., $p=0.9$). Developing a tighter theoretical framework that fully captures the asymmetric roles of factors remains an area for future exploration.

\bibliography{iclr2027_conference}
\bibliographystyle{iclr2027_conference}

% ---------------------------- Appendix ------------------------------
\newpage
\appendix

\section*{Appendix}
\startcontents[sections]
\printcontents[sections]{l}{1}{\setcounter{tocdepth}{2}}
\vspace{1cm} % Adds a little breathing room before the first section

\newpage
\section{Proof of the Theoretical Results} \label{app:all_proof}
\subsection{Proof of Proposition \ref{prop:exact}}
\label{app:proof_exact_agg}

Expanding the product of the aggregated factors yields
\begin{equation}
\begin{aligned}
&
% \mathbf B^{t+1}\mathbf A^{t+1}
% =
\left(\mathbf B^t+\sum_i\frac{w_i}{q}\Delta\mathbf B_i^t\mathbf Z^t\right)
\left(\mathbf A^t+\sum_i\frac{w_i}{q}(\mathbf I_R-\mathbf Z^t)\Delta\mathbf A_i^t\right) \\
=&
\mathbf B^t\mathbf A^t
+\sum_i\frac{w_i}{q}\Delta\mathbf B_i^t\mathbf Z^t\mathbf A^t
+\sum_i\frac{w_i}{q}\mathbf B^t(\mathbf I_R-\mathbf Z^t)\Delta\mathbf A_i^t
+\sum_{i,k}\frac{w_iw_k}{q^2}\Delta\mathbf B_i^t\mathbf Z^t(\mathbf I_R-\mathbf Z^t)\Delta\mathbf A_k^t \\
=&
\mathbf B^t\mathbf A^t
+\sum_i\frac{w_i}{q}\Delta\mathbf B_i^t\mathbf Z^t\mathbf A^t
+\sum_i\frac{w_i}{q}\mathbf B^t(\mathbf I_R-\mathbf Z^t)\Delta\mathbf A_i^t.
\label{eq:server_product_expansion}
\end{aligned}    
\end{equation}
Each client endpoint satisfies
\begin{equation}
     \mathbf B_i^{t+1}\mathbf A_i^{t+1}
    = (\mathbf{B}^t + \Delta\mathbf{B}_i^t \mathbf{Z}^t)
      (\mathbf{A}^t + (\mathbf{I} - \mathbf{Z}^t)\Delta\mathbf{A}_i^t)
    =
    \mathbf B^t\mathbf A^t
    +\Delta\mathbf B_i^t\mathbf Z^t\mathbf A^t
    +\mathbf B^t(\mathbf I_R-\mathbf Z^t)\Delta\mathbf A_i^t,   
\end{equation}
Subtracting $\mathbf{B}^{t}\mathbf{A}^{t}$ from each client endpoint product and summing over $i\in\mathcal{S}_t$ with weights $w_i/q$ yields exactly the two linear terms in the expansion of the aggregated factors. This completes the proof.

% ------------------- Proof of Proposition (Hete) Exact Aggregation -----------------
\subsection{Proof of Proposition \ref{prop:heterogeneous_exact}}
\label{app:proof_hete_exact_agg}
Expanding the product of the aggregated factors therefore yields
\begin{equation}
\begin{aligned}
&
% \mathbf B^{t+1}\mathbf A^{t+1}
% =
\left(\mathbf B^t+\sum_i\frac{w_i}{q}\Delta\widetilde{\mathbf B}_i^t(\mathbf E_i^t)^\top\right)
\left(\mathbf A^t+\sum_i\frac{w_i}{q}\mathbf E_i^t\Delta\widetilde{\mathbf A}_i^t\right) \\
=&
\mathbf B^t\mathbf A^t
+\sum_i\frac{w_i}{q}\Delta\widetilde{\mathbf B}_i^t(\mathbf E_i^t)^\top\mathbf A^t
+\sum_i\frac{w_i}{q}\mathbf B^t\mathbf E_i^t\Delta\widetilde{\mathbf A}_i^t
+\sum_{i,k}\frac{w_iw_k}{q^2}\Delta\widetilde{\mathbf B}_i^t(\mathbf E_i^t)^\top\mathbf Z^t(\mathbf I_R-\mathbf Z^t)\mathbf E_k^t\Delta\widetilde{\mathbf A}_k^t \\
=&
\mathbf B^t\mathbf A^t
+\sum_i\frac{w_i}{q}\Delta\widetilde{\mathbf B}_i^t(\mathbf E_i^t)^\top\mathbf A^t
+\sum_i\frac{w_i}{q}\mathbf B^t\mathbf E_i^t\Delta\widetilde{\mathbf A}_i^t.
\label{eq:heterogeneous_server_product_expansion}
\end{aligned}
\end{equation}

For the local models, each client's compact endpoint satisfies
\begin{equation}
\begin{aligned}
     \widetilde{\mathbf B}_i^{t,K}\widetilde{\mathbf A}_i^{t,K}
    &= (\mathbf B^t\mathbf E_i^t + \Delta\widetilde{\mathbf B}_i^t)
       ((\mathbf E_i^t)^\top\mathbf A^t + \Delta\widetilde{\mathbf A}_i^t) \\
    &= \mathbf B^t\mathbf H_i^t\mathbf A^t
    +\Delta\widetilde{\mathbf B}_i^t(\mathbf E_i^t)^\top\mathbf A^t
    +\mathbf B^t\mathbf E_i^t\Delta\widetilde{\mathbf A}_i^t
    +\Delta\widetilde{\mathbf B}_i^t(\mathbf E_i^t)^\top\mathbf Z^t(\mathbf I_R-\mathbf Z^t)\mathbf E_i^t\Delta\widetilde{\mathbf A}_i^t \\
    &= \mathbf B^t\mathbf H_i^t\mathbf A^t
    +\Delta\widetilde{\mathbf B}_i^t(\mathbf E_i^t)^\top\mathbf A^t
    +\mathbf B^t\mathbf E_i^t\Delta\widetilde{\mathbf A}_i^t.
\end{aligned}
\end{equation}
Subtracting $\mathbf B^t\mathbf H_i^t\mathbf A^t$ and taking the weighted summation over all clients yields exactly the two linear terms in \eqref{eq:heterogeneous_server_product_expansion}. Substituting this weighted sum back into \eqref{eq:heterogeneous_server_product_expansion} reconstructs the right-hand side of \eqref{eq:heterogeneous_exactness}, which concludes the proof.

\subsection{Assumptions}
\label{app:assumptions}

\setcounter{assumption}{0}
\renewcommand{\theassumption}{A.\arabic{assumption}}

We adopt standard assumptions widely used in the analysis of non-convex federated optimization.

\begin{assumption}\label{assump:smooth}
For each client $i$ and activation mask
$\mathbf H \in \operatorname{supp}(\mathcal H_i)$,
the local objective $f_i^{\mathbf H}(\mathbf X)$ is differentiable and $L$-smooth over a region $\mathcal C$ containing the iterates; that is, there exists a constant $L>0$ such that for all $\mathbf X,\mathbf X'\in\mathcal C$,
\begin{equation}
   \|\nabla f_i^{\mathbf H}(\mathbf X)
-\nabla f_i^{\mathbf H}(\mathbf X')\|
\le L\|\mathbf X-\mathbf X'\|,
\quad\forall i. 
\end{equation}
Consequently, the global objective $\Phi(\mathbf X)$ is also
$L$-smooth on $\mathcal C$. Moreover, there exists a finite lower bound
$\Phi^\star$ such that $\Phi(\mathbf X)\ge \Phi^\star$ for all
$\mathbf X\in\mathcal C$.
\end{assumption}

\begin{assumption}\label{assump:unbias}
For each client $i$ and any activation mask $\mathbf H \in \operatorname{supp}(\mathcal H_i)$, the stochastic gradient is unbiased, i.e.,
\begin{equation}
    \mathbb E_{\xi\sim\mathcal D_i}
\left[
\nabla_{\mathbf X}
\ell(\mathbf W_0+\mathbf B\mathbf H\mathbf A;\xi)
\right]
=
\nabla f_i^{\mathbf H}(\mathbf X).
\end{equation}
Furthermore, there exist constants $\sigma_i^2,\tau_i^2\ge0$ bounding the variances of minibatch sampling and mask sampling, respectively, and a uniformly bounded second moment $M>0$:
\begin{align}
\sup_{\substack{\mathbf X\in\mathcal C,\,
\mathbf H\in\operatorname{supp}(\mathcal H_i)}}
\mathbb E_{\xi\sim\mathcal D_i}
\left[
\left\|
\nabla_{\mathbf X}
\ell(\mathbf W_0+\mathbf B\mathbf H\mathbf A;\xi)
-\nabla f_i^{\mathbf H}(\mathbf X)
\right\|^2
\right]
&\le\sigma_i^2,
 \\
\sup_{\mathbf X\in\mathcal C}
\mathbb E_{\mathbf H\sim\mathcal H_i}
\left[
\left\|
\nabla f_i^{\mathbf H}(\mathbf X)
-\nabla\phi_i(\mathbf X)
\right\|^2
\right]
&\le\tau_i^2,
\\
\sup_{\substack{\mathbf X\in\mathcal C,\,
\mathbf H\in\operatorname{supp}(\mathcal H_i)}}
\mathbb E_{\xi\sim\mathcal D_i}
\left[
\left\|
\nabla_{\mathbf X}
\ell(\mathbf W_0+\mathbf B\mathbf H\mathbf A;\xi)
\right\|^2
\right]
&\le M. 
\end{align}
\end{assumption}

Assumption~\ref{assump:smooth} is standard in optimization literature
\citep{bottou2018optimization, fang2026fslora}.
Assumption~\ref{assump:unbias} is commonly adopted in FL to bound sampling randomness and stochastic-gradient moments
\citep{fang2026fslora}.

\subsection{Lemma~\ref{lem:compact_unbiased} and Its Proof}
\label{app:proof_lemma}

\setcounter{lemma}{0}
\renewcommand{\thelemma}{A.\arabic{lemma}}

\begin{lemma}
\label{lem:compact_unbiased}
Under Assumption~\ref{assump:unbias}, we have
\begin{equation}
\label{eq:compact_unbiased}
\mathbb E[\mathbf G_B^t]
=
\nabla_{\mathbf B}\Phi(\mathbf X^t),
\qquad
\mathbb E[\mathbf G_A^t]
=
\nabla_{\mathbf A}\Phi(\mathbf X^t),
\end{equation}
where $\mathbf G_B^t = \frac{1}{p}\mathbf S_B^t\mathbf Z^t$ and $\mathbf G_A^t = \frac{1}{1-p}(\mathbf I_R-\mathbf Z^t)\mathbf S_A^t$ denote the scaled aggregated gradients, and $\mathbf S_{Y}^t = \sum_{i\in\mathcal S_t}\frac{w_i}{q} \left. \nabla_{\mathbf Y} \ell(\mathbf W_0+\mathbf B\mathbf H_i^t\mathbf A;\xi_i^{t,0}) \right|_{\mathbf X=\mathbf X^t}$ are the pre-selection aggregated gradients at the common round-start iterate $\mathbf X^t$ for $Y \in \{B, A\}$, with local minibatch sample $\xi_i^{t,0}\sim\mathcal D_i$.
\end{lemma}

\begin{proof}
Let $I_i^t=\mathbf 1\{i\in\mathcal S_t\}$,
$q=m/N$, and $\alpha_i^t=w_iI_i^t/q$.
Uniform client sampling gives $\mathbb E[I_i^t]=q$.
By the independence of client sampling and activation masks, we have 
\begin{equation}
\mathbb E[\mathbf S_B^t]
=
\sum_{i}\frac{w_i}{q}\mathbb E[I_i^t]
\mathbb E_{\mathbf H\sim\mathcal H_i}
\bigl[\nabla_{\mathbf B}f_i^{\mathbf H}(\mathbf X^t)\bigr]
=
\nabla_{\mathbf B}\Phi(\mathbf X^t).
\end{equation}
Similarly, $\mathbb E[\mathbf S_A^t]
=
\nabla_{\mathbf A}\Phi(\mathbf X^t)$ holds.
Since $\mathbf Z^t$ is conditionally independent of
$\mathbf S_{B}^t,\mathbf S_{A}^t$ and
$\mathbb E[\mathbf Z^t]=p\mathbf I_R$,
iterated expectations yield
\begin{equation}
\begin{aligned}
\mathbb E[\mathbf G_B^t]
&=
\frac{1}{p}\mathbb E\!\left[
\mathbf S_B^t
\mathbb E[\mathbf Z^t\mid\mathbf S_B^t]
\right]
=
\nabla_{\mathbf B}\Phi(\mathbf X^t),
\\
\mathbb E[\mathbf G_A^t]
&=
\frac{1}{1-p}\mathbb E\!\left[
\mathbb E[\mathbf I_R-\mathbf Z^t
\mid\mathbf S_A^t]
\mathbf S_A^t
\right]
=
\nabla_{\mathbf A}\Phi(\mathbf X^t).
\end{aligned}
\end{equation}
This leads to~\eqref{eq:compact_unbiased}, which concludes the proof.
\end{proof}

% ----------------------- Proof of Theorem ---------------------
\subsection{Proof of Theorem~\ref{thm:convergence}}
\label{app:proof_convergence}
% We first show the following lemma characterizes the conditional mean and second moment of the scaled gradients, and then we prove the convergence behavior of \texttt{CFLoRA}.

\begin{proof}
Let
$
p_{\min}=\min\{p,1-p\},
\kappa=p_{\min}^{-1},
\gamma=\eta K,
\alpha_i^t=\frac{w_i}{q}\mathbf 1\{i\in\mathcal S_t\}, I_i^t=\mathbf 1\{i\in\mathcal S_t\}$,
where $p\in(0,1)$ is fixed and $q=m/N$.

\textbf{Step 1: Selection and client-sampling identities.}
To formalize the complementary channel selection mechanism in \texttt{CFLoRA}, let $\mathbf{u} = (\mathbf{u}_B, \mathbf{u}_A)$ denote an arbitrary partitioned variable in the parameter or gradient space. We define the random scaling operator $\mathcal{R}_t$ as:
$\mathcal{R}_t(\mathbf{u})
=
\left(
\frac{\mathbf{u}_B\mathbf{Z}^t}{p},
\frac{(\mathbf{I}_R-\mathbf{Z}^t)\mathbf{u}_A}{1-p}
\right)$.
By the definition of $\kappa = p_{\min}^{-1}$, for any $\mathbf{u}$, the operator's norm is bounded:
\begin{equation}\label{eq:proof_pp_operator}
\|\mathcal{R}_t(\mathbf{u})\| \le \kappa\|\mathbf{u}\|.
\end{equation}
Furthermore, for any $\mathcal{F}_t$-measurable $\mathbf{u}$, taking the expectation over the random mask yields unbiasedness and a bounded second moment:
\begin{equation}\label{eq:proof_pp_operator_moment}
\begin{aligned}
\mathbb{E}[\mathcal{R}_t(\mathbf{u})] 
= \mathbf{u}, 
\qquad
\mathbb{E}\|\mathcal{R}_t(\mathbf{u})\|^2 
= \frac{\|\mathbf{u}_B\|_F^2}{p} + \frac{\|\mathbf{u}_A\|_F^2}{1-p} \le \kappa\|\mathbf{u}\|^2.
\end{aligned}
\end{equation}

% Since $\mathbf{Z}^t$ is generated independently of the client sampling and local minibatches at round $t$, these identities continue to hold when conditioning on these concurrent random events.

% These identities also hold after conditioning on randomness independent
% of $\mathbf Z^t$ given $\mathcal F_t$.

Uniform client sampling gives $\mathbb E[I_i^t]=q$ and $\mathbb E[I_i^tI_j^t]
= \frac{m(m-1)}{N(N-1)}$,
and hence for vectors $\mathbf y_i$, we have
\begin{equation}\label{eq:proof_pp_sampling}
\begin{aligned}
\mathbb E\!\left[\sum_i\alpha_i^t\mathbf y_i\right]
&=\sum_iw_i\mathbf y_i,
\\
\mathbb E\left\|\sum_i\alpha_i^t\mathbf y_i\right\|^2
&=
\|\sum_iw_i\mathbf y_i\|^2
+
\frac{N-m}{m(N-1)}\left(
N\sum_iw_i^2\|\mathbf y_i\|^2
-
\|\sum_iw_i\mathbf y_i\|^2
\right).
\end{aligned}
\end{equation}
In particular, define $\omega_t=\sum_i\alpha_i^t$, we have
\begin{equation}\label{eq:proof_pp_weight_moments}
\mathbb E[\omega_t]=1,
\qquad
\mathbb E[\omega_t^2]
=
1+\frac{N-m}{m(N-1)}\left(N\sum_iw_i^2-1\right)
=
\chi_m.
\end{equation}

\textbf{Step 2: Endpoint decomposition and sampling noise.}
The lifted local updates satisfy
$\mathbf X_i^{t,0}=\mathbf X^t$ and
$\mathbf X_i^{t,k+1}
=\mathbf X_i^{t,k}-\eta\mathcal R_t(\mathbf g_i^{t,k})$,
where $\mathbf g_i^{t,k}$ is the stochastic gradient of
$f_i^{\mathbf H_i}$ at $\mathbf X_i^{t,k}$.
The server update becomes
$\mathbf X^{t+1}=\mathbf X^t-\gamma\mathbf v^t$,
with $\mathbf v^t$ defined as 
\begin{equation}
    \begin{aligned}
\mathbf v^t
&=
\mathcal R_t\left(
\sum_i\alpha_i^t\mathbf \nabla f_i^{\mathbf H}(\mathbf X_i^{t,0})
\right)+
\frac1K\sum_{k=0}^{K-1}\sum_i\alpha_i^t
\mathcal R_t(\nabla f_i^{\mathbf H}(\mathbf X_i^{t,k})-\mathbf \nabla f_i^{\mathbf H}(\mathbf X_i^{t,0})),
\\
&+
\frac1K\sum_{k=0}^{K-1}\sum_i\alpha_i^t
\mathcal R_t((\mathbf g_i^{t,k}-\nabla f_i^{\mathbf H}(\mathbf X_i^{t,k}))).
\end{aligned}
\end{equation}
Independent activation masks and~\eqref{eq:proof_pp_sampling} yield
\begin{equation}
\begin{aligned}
\mathbb E
\left\|\sum_i\alpha_i^t\mathbf \nabla f_i^{\mathbf H}(\mathbf X_i^{t,0})\right\|^2
&=
\mathbb E
\left\|\sum_i\alpha_i^t\nabla\phi_i(\mathbf X^t)\right\|^2
+
\frac1q\sum_iw_i^2
\mathbb E
\|\mathbf \nabla f_i^{\mathbf H}(\mathbf X_i^{t,0})-\nabla\phi_i(\mathbf X^t)\|^2
\\
&\le
\|\nabla\Phi(\mathbf X^t)\|^2+\frac{N-m}{m(N-1)}\zeta_w^2+\frac{\tau_w^2}{q}.
\end{aligned}
\end{equation}
where $\zeta_w^2
=
\sup_{\mathbf X\in\mathcal C}
\left\{
N\sum_{i}w_i^2
\|\nabla\phi_i(\mathbf X)\|^2
-
\|\nabla\Phi(\mathbf X)\|^2
\right\}$ and $\tau_w^2=\sum_{i=1}^{N}w_i^2\tau_i^2$. Since the round-start gradients are independent of $\mathbf Z^t$,
\begin{equation}\label{eq:proof_pp_roundstart}
\begin{aligned}
    \mathbb E\left[\mathcal R_t\left( \sum_i\alpha_i^t\mathbf \nabla f_i^{\mathbf H}(\mathbf X_i^{t,0}) \right)\right] & =\nabla\Phi(\mathbf X^t),
\\
\mathbb E\left\|\mathcal R_t\left( \sum_i\alpha_i^t\mathbf \nabla f_i^{\mathbf H}(\mathbf X_i^{t,0}) \right)\right\|^2
& \le
\kappa\left(
\|\nabla\Phi(\mathbf X^t)\|^2+\frac{N-m}{m(N-1)}\zeta_w^2+\frac{\tau_w^2}{q}
\right).
\end{aligned}
\end{equation}

Let $\mathcal{F}_{t,k}$ be the filtration extending $\mathcal{F}_t$ with all randomness up to local step $k$.
Fresh minibatches give
$\mathbb E[
(\mathbf g_i^{t,k}-\nabla f_i^{\mathbf H}(\mathbf X_i^{t,k}))\mid\mathcal F_{t,k}
]=0$
and 
$\mathbb E[
\|(\mathbf g_i^{t,k}-\nabla f_i^{\mathbf H}(\mathbf X_i^{t,k}))\|^2\mid\mathcal F_{t,k}
]
\le\sigma_i^2.$

Noise cross terms vanish across steps by conditional centering and
across clients by conditional independence. Since
$\mathbb E[(\alpha_i^t)^2]=w_i^2/q$,
\begin{equation}\label{eq:proof_pp_noise}
\begin{aligned}
    \mathbb E[\frac1K\sum_{k=0}^{K-1}\sum_i\alpha_i^t \mathcal R_t((\mathbf g_i^{t,k}-\nabla f_i^{\mathbf H}(\mathbf X_i^{t,k})))]&=0,
\\
\mathbb E\|\frac1K\sum_{k=0}^{K-1}\sum_i\alpha_i^t \mathcal R_t((\mathbf g_i^{t,k}-\nabla f_i^{\mathbf H}(\mathbf X_i^{t,k})))\|^2
&\le\frac{\kappa^2\sigma_w^2}{qK}, 
\\
\mathbb E\left\langle\mathcal R_t\left( \sum_i\alpha_i^t\mathbf \nabla f_i^{\mathbf H}(\mathbf X_i^{t,0}) \right),\frac1K\sum_{k=0}^{K-1}\sum_i\alpha_i^t \mathcal R_t((\mathbf g_i^{t,k}-\nabla f_i^{\mathbf H}(\mathbf X_i^{t,k})))\right\rangle&=0.
\end{aligned}
\end{equation}
The last identity uses that $\mathcal R_t\left( \sum_i\alpha_i^t\mathbf \nabla f_i^{\mathbf H}(\mathbf X_i^{t,0}) \right)$ is fixed given the selected
set and round masks.

\textbf{Step 3: Local drift.}
Write
$\mathbb E_{t,S}[\cdot]
=\mathbb E[\cdot\mid\mathcal F_t,\mathcal S_t]$.
Assumption~\ref{assump:unbias}
and~\eqref{eq:proof_pp_operator} imply
\begin{equation}\label{eq:proof_pp_displacement}
\begin{aligned}
\left(
\mathbb E_{t,S}
\|\mathbf X_i^{t,k}-\mathbf X^t\|^2
\right)^{1/2}
&\le
\eta\sum_{s=0}^{k-1}
\left(
\mathbb E_{t,S}
\|\mathcal R_t(\mathbf g_i^{t,s})\|^2
\right)^{1/2}
\\
&\le\eta k\kappa\sqrt M.
\end{aligned}
\end{equation}
Smoothness and the conditional $L^2$ triangle inequality give
\begin{equation}\label{eq:proof_pp_drift_second}
\begin{aligned}
&\left(\mathbb E_{t,S}\|\frac1K\sum_{k=0}^{K-1}\sum_i\alpha_i^t \mathcal R_t(\nabla f_i^{\mathbf H}(\mathbf X_i^{t,k})-\mathbf \nabla f_i^{\mathbf H}(\mathbf X_i^{t,0}))\|^2\right)^{1/2}\\
&\le
\frac{\kappa L}{K}\sum_{k,i}\alpha_i^t
\left(
\mathbb E_{t,S}
\|\mathbf X_i^{t,k}-\mathbf X^t\|^2
\right)^{1/2}
\\
&\le
\frac{\omega_t}{2}L\eta\kappa^2\sqrt M(K-1).
\end{aligned}
\end{equation}
For any $\mathcal F_t$-measurable unit vector $\mathbf z$,
self-adjointness and Cauchy--Schwarz yield
\begin{equation}\label{eq:proof_pp_drift_bias}
\begin{aligned}
&\left|
\langle\mathbf z,\mathbb E_{t,S}[\frac1K\sum_{k=0}^{K-1}\sum_i\alpha_i^t \mathcal R_t(\nabla f_i^{\mathbf H}(\mathbf X_i^{t,k})-\mathbf \nabla f_i^{\mathbf H}(\mathbf X_i^{t,0}))]\rangle
\right|
\\
&\le
\frac1K\sum_{k,i}\alpha_i^t
\left(
\mathbb E_{t,S}\|\mathcal R_t(\mathbf z)\|^2
\right)^{1/2}
\left(
\mathbb E_{t,S}
\|\nabla f_i^{\mathbf H}(\mathbf X_i^{t,k})-\mathbf \nabla f_i^{\mathbf H}(\mathbf X_i^{t,0})\|^2
\right)^{1/2}
\\
&\le
\frac{\omega_t}{2}L\eta\kappa^{3/2}\sqrt M(K-1).
\end{aligned}
\end{equation}
Define
$d
=
\frac14L^2\eta^2\kappa^3M(K-1)^2
=
\frac14L^2\gamma^2\kappa^3M
\left(1-\frac1K\right)^2$, 
averaging over $\mathcal S_t$
and using~\eqref{eq:proof_pp_weight_moments},
then taking the supremum over unit $\mathbf z$, gives
\begin{equation}\label{eq:proof_pp_drift}
\begin{aligned}
    \|\mathbb E[\frac1K\sum_{k=0}^{K-1}\sum_i\alpha_i^t \mathcal R_t(\nabla f_i^{\mathbf H}(\mathbf X_i^{t,k})-\mathbf \nabla f_i^{\mathbf H}(\mathbf X_i^{t,0}))]\|^2 & \le d,
\\
    \mathbb E\|\frac1K\sum_{k=0}^{K-1}\sum_i\alpha_i^t \mathcal R_t(\nabla f_i^{\mathbf H}(\mathbf X_i^{t,k})-\mathbf \nabla f_i^{\mathbf H}(\mathbf X_i^{t,0}))\|^2 & \le\chi_m\kappa d.
\end{aligned}
\end{equation}

\textbf{Step 4: Descent and telescoping.}
By~\eqref{eq:proof_pp_roundstart},
\eqref{eq:proof_pp_noise},
and~\eqref{eq:proof_pp_drift},
\begin{equation}\label{eq:proof_pp_direction}
\begin{aligned}
\mathbb E\|\mathbf v^t\|^2
\le&
2\mathbb E\|\mathcal R_t\left( \sum_i\alpha_i^t\mathbf \nabla f_i^{\mathbf H}(\mathbf X_i^{t,0}) \right)+\frac1K\sum_{k=0}^{K-1}\sum_i\alpha_i^t \mathcal R_t((\mathbf g_i^{t,k}-\nabla f_i^{\mathbf H}(\mathbf X_i^{t,k})))\|^2 \\
&+
2\mathbb E\|\frac1K\sum_{k=0}^{K-1}\sum_i\alpha_i^t \mathcal R_t(\nabla f_i^{\mathbf H}(\mathbf X_i^{t,k})-\mathbf \nabla f_i^{\mathbf H}(\mathbf X_i^{t,0}))\|^2
\\
\le &
2\kappa\|\nabla\Phi(\mathbf X^t)\|^2
+
2\mathcal V_K
+
2\chi_m\kappa d.
\end{aligned}
\end{equation}
Moreover,
$\mathbb E[\mathbf v^t]
=\nabla\Phi(\mathbf X^t)+\mathbb E[\frac1K\sum_{k=0}^{K-1}\sum_i\alpha_i^t \mathcal R_t(\nabla f_i^{\mathbf H}(\mathbf X_i^{t,k})-\mathbf \nabla f_i^{\mathbf H}(\mathbf X_i^{t,0}))]$
and $ab\le a^2/4+b^2$ imply
\begin{equation}\label{eq:proof_pp_alignment}
-\langle\nabla\Phi(\mathbf X^t),\mathbb E[\mathbf v^t]\rangle
\le
-\frac34\|\nabla\Phi(\mathbf X^t)\|^2+d.
\end{equation}
By $L$-smoothness and $\gamma\le(4L\kappa)^{-1}$,
\begin{align*}
\mathbb E[\Phi(\mathbf X^{t+1})]
&\le
\Phi(\mathbf X^t)
-
\gamma\langle\nabla\Phi(\mathbf X^t),\mathbb E[\mathbf v^t]\rangle
+
\frac{L\gamma^2}{2}\mathbb E\|\mathbf v^t\|^2
\\
&\le
\Phi(\mathbf X^t)
-
\gamma\left(\frac34-L\gamma\kappa\right)\|\nabla\Phi(\mathbf X^t)\|^2
+
L\gamma^2\mathcal V_K
+
\gamma(1+L\gamma\kappa\chi_m)d
\\
&\le
\Phi(\mathbf X^t)
-
\frac\gamma2\|\nabla\Phi(\mathbf X^t)\|^2
+
L\gamma^2\mathcal V_K
+
\gamma\left(1+\frac{\chi_m}{4}\right)d.
\end{align*}
Taking total expectations, summing over $t$, and using
$\Phi\ge\Phi^\star$ yields
\begin{equation}\label{eq:proof_pp_base_bound}
\begin{aligned}
\frac1T\sum_{t=0}^{T-1}
\mathbb E\|\nabla\Phi(\mathbf X^t)\|^2
&\le
\frac{2\Delta_0}{\gamma T}
+
2L\gamma\mathcal V_K
+
\left(2+\frac{\chi_m}{2}\right)d
\\
&=
\frac{2\Delta_0}{\gamma T}
+
2L\gamma\mathcal V_K
+
\mathcal D_K\gamma^2,
\end{aligned}
\end{equation}
where
\[
\mathcal V_K
=
\kappa\left(
\frac{\tau_w^2}{q}+\frac{N-m}{m(N-1)}\zeta_w^2
\right)
+
\frac{\kappa^2\sigma_w^2}{qK},
\qquad
\mathcal D_K
=
\frac{4+\chi_m}{8}L^2\kappa^3M
\left(1-\frac1K\right)^2.
\]

\textbf{Step 5: Learning-rate choice.}
For $\gamma_T=K\eta_T$
in~\eqref{eq:explicit_learning_rate},
omitting candidates with zero denominators,
\[
\frac{2\Delta_0}{\gamma_TT}
=
\max\left\{
\frac{8L\kappa\Delta_0}{T},
2\sqrt{\frac{L\Delta_0\mathcal V_K}{T}},
\left(
\frac{4\Delta_0^2\mathcal D_K}{T^2}
\right)^{1/3}
\right\},
\]
and
\[
2L\gamma_T\mathcal V_K
\le
2\sqrt{\frac{L\Delta_0\mathcal V_K}{T}},
\qquad
\mathcal D_K\gamma_T^2
\le
\left(
\frac{4\Delta_0^2\mathcal D_K}{T^2}
\right)^{1/3}.
\]
Substituting into~\eqref{eq:proof_pp_base_bound}
and bounding the maximum by the sum
establishes~\eqref{eq:convergence_bound}, which completes the proof.
\end{proof}

\subsection{Proof of Corollary~\ref{cor:homogeneous}}
\label{app:proof_homogeneous}

\begin{proof}
When $r_i=R$ for every client,
$\Phi=F$ and $\tau_w^2=0$.
Substituting $\gamma=(4L\kappa\sqrt T)^{-1}$
into~\eqref{eq:proof_pp_base_bound} gives
\begin{align*}
\frac1T\sum_{t=0}^{T-1}
\mathbb E\|\nabla F(\mathbf X^t)\|^2
&\le
\frac{
8L\kappa\Delta_0
+\frac{N-m}{m(N-1)}\zeta_w^2/2
+\kappa\sigma_w^2/(2qK)
}{\sqrt T}
\\
&\quad+
\frac{(4+\chi_m)\kappa M}{128T}
\left(1-\frac1K\right)^2,
\end{align*}
which is~\eqref{eq:homogeneous_rate}.
\end{proof}

\subsection{Insights from Theorem \ref{thm:convergence}.}
\label{app:insights_convergence}
\paragraph{Asymptotic Convergence and Exact Aggregation.}Theorem~\ref{thm:convergence} establishes that, under an appropriately tuned learning rate, \texttt{CFLoRA} achieves a sublinear convergence rate of $\mathcal{O}(T^{-1/2})$ to a stationary point of the surrogate objective. Crucially, this bound isolates the effects of the learning rate and statistical noise without suffering from any non-vanishing error floor. By employing complementary-channel selection, \texttt{CFLoRA} achieves exact aggregation, thereby completely eliminating this bilinear aggregation error from the asymptotic rate. The parameter $\kappa = 1/\min\{p, 1-p\}$ explicitly captures the variance expansion cost of alternating between the factor matrices $\mathbf{A}$ and $\mathbf{B}$.

\paragraph{Components in Variance ($\mathcal{V}_K$).}The term $\mathcal{V}_K$ encapsulates three distinct sources of statistical variance in the federated system, each scaled by the protocol's alternating selection factor $\kappa$:

\begin{itemize}
\item \textbf{Activation-mask variability ($\tau_w^2$):} This variance arises from the heterogeneity of client rank budgets. Because clients randomly sample sub-matrices corresponding to their local capacity, their updates fluctuate around the true gradient of $\Phi$. If all clients possess the capacity to use the full server rank (i.e., $r_i = R$ and $\mathbf{H}_i^t = \mathbf{I}_R$), this term vanishes.
\item \textbf{Client-sampling variance ($\zeta_w^2$):} This term accounts for the data heterogeneity across different clients under partial participation ($m < N$). The coefficient $\frac{N-m}{m(N-1)}$ reflects standard sampling without replacement; thus, if the server aggregates from all clients in every round ($m = N$), this inter-client variance vanishes.
\item \textbf{Minibatch noise ($\sigma_w^2$):} The standard stochastic gradient noise from local data sampling is effectively suppressed by the number of local gradient steps executed across participating clients ($K$).
\end{itemize}
\paragraph{Tradeoffs of Local Training Steps ($K$).} Theorem \ref{thm:convergence} reveals a fundamental theoretical tradeoff governed by the number of local steps $K$, explicitly controlled by the balance between $\mathcal{V}_K$ and the local drift term $\mathcal{D}_K$. Executing more local steps reduces the impact of minibatch noise (via the $\frac{\kappa^2\sigma_w^2}{qK}$ term in $\mathcal{V}_K$). However, this variance reduction comes at the cost of increased local drift across clients ($\mathcal{D}_K$), which scales with $(1 - 1/K)^2$. Because this structural drift vanishes entirely when $K=1$, the choice of $K$ should balance statistical minibatch variance and client local drifts. Note that the activation-mask variability ($\tau_w^2$) remains unaffected by $K$, as the masks are drawn only once per communication round.

\paragraph{Heterogeneous Ranks and the Surrogate Objective ($\Phi$).} Since clients with low rank budgets evaluate gradients using masked sub-matrices, their updates are structurally biased with respect to the full-adapter objective $F$. Consequently, the theoretical guarantee concerns the surrogate objective $\Phi$, which incorporates these heterogeneous client ranks through the activation distributions $\{\mathcal H_i\}_{i=1}^{N}$. When every client uses the full server rank ($\mathbf H_i^t=\mathbf I_R$), $\Phi$ reduces to $F$.

%

% For equal aggregation weights and uniform variance bounds
% $\sigma_i^2\leq\sigma^2$ and $\tau_i^2\leq\tau^2$,
% \[
% V_K\leq
% \frac{\kappa\tau^2}{N}
% +\frac{\kappa^2\sigma^2}{NK}.
% \]
% However, these reductions in statistical terms do not imply
% unconditional linear speedup, since the drift term must
% also be controlled.

%\paragraph{Heterogeneous ranks and adaptive selection.}  Because clients with restricted rank budgets evaluate gradients using masked sub-matrices, their updates are structurally biased with respect to the full-adapter objective $F$. Consequently, our theoretical guarantee must concern a surrogate objective $\Phi$, which formally incorporates these heterogeneous client ranks through the activation distributions $\{\mathcal H_i\}_{i=1}^{N}$. When every client possesses the capacity to use the full server rank (i.e., $\mathbf H_i^t=\mathbf I_R$), activation-mask variability vanishes and $\Phi$ exactly reduces to $F$. However, for systems with smaller client ranks, optimizing the surrogate $\Phi$ represents the best achievable model adaptation, as stationarity of $\Phi$ does not generally imply stationarity of $F$.

%The theorem also accommodates history-dependent selection probabilities, including the contribution-aware mechanism, provided that $p$ is clipped to $[p_{\min},1-p_{\min}]$. The dependence on $\kappa$ makes explicit the cost of allowing either factor's selection probability to approach zero.

% \newpage
% \section{Algorithm of \texttt{CFLoRA}}
% \label{app:algorithm}

\section{Algorithm of \texttt{CFLoRA}}
\label{app:algorithm}

\setcounter{algocf}{0}
\renewcommand{\thealgocf}{B.\arabic{algocf}}

\begin{algorithm}[h]
    \caption{\texttt{CFLoRA} with Homogeneous Client Ranks}
    \label{alg:CFLoRA-homogeneous}
    \LinesNumbered
    \KwIn{$\mathbf W_0$, $(\mathbf B_{\mathrm{init}},
    \mathbf A_{\mathrm{init}})$, $R$,
    $\{w_i\}_{i=1}^{N}$, $m$,
    $T$, $K$, $\eta$, $p\in(0,1)$}
    \KwOut{$(\mathbf B^T,\mathbf A^T)$}

    $\mathbf B^0\leftarrow\mathbf B_{\mathrm{init}}$,
    $\mathbf A^0\leftarrow\mathbf A_{\mathrm{init}}$\;
    $q\leftarrow m/N$,
    $\eta_B\leftarrow\eta/p$,
    $\eta_A\leftarrow\eta/(1-p)$\;

    \For{$t\leftarrow 0$ \KwTo $T-1$}{
        Sample $m$ clients from $\{1,\ldots,N\}$\ to form $\mathcal S^t$\;
        Draw $z_j^t\sim\operatorname{Bernoulli}(p)$
        independently for $j=1,\ldots,R$\;
        $\mathbf Z^t\leftarrow
        \operatorname{diag}(z_1^t,\ldots,z_R^t)$\;
        Broadcast $(\mathbf B^t,\mathbf A^t,\mathbf Z^t)$
        to clients in $\mathcal S^t$\;

        \For{$i\in\mathcal S^t$ \textnormal{ in parallel}}{
            $\mathbf B_i^{t,0}\leftarrow\mathbf B^t$,
            $\mathbf A_i^{t,0}\leftarrow\mathbf A^t$\;

            \For{$k\leftarrow 0$ \KwTo $K-1$}{
                Compute gradients
                $(\mathbf g_{B,i}^{t,k},
                \mathbf g_{A,i}^{t,k})$ of
                $\ell\!\left(
                \mathbf W_0+
                \mathbf B_i^{t,k}\mathbf A_i^{t,k};
                \xi_i^{t,k}\right)$
                with respect to the local factors\;
                $\mathbf B_i^{t,k+1}
                \leftarrow\mathbf B_i^{t,k}
                -\eta_B\mathbf g_{B,i}^{t,k}
                \mathbf Z^t$\;
                $\mathbf A_i^{t,k+1}
                \leftarrow\mathbf A_i^{t,k}
                -\eta_A (\mathbf{I}_R - \mathbf Z^t)\mathbf g_{A,i}^{t,k}
                $\;
            }

            $\Delta\mathbf B_i^t
            \leftarrow\mathbf B_i^{t,K}
            -\mathbf B_i^{t,0}$\;
            $\Delta\mathbf A_i^t
            \leftarrow\mathbf A_i^{t,K}
            -\mathbf A_i^{t,0}$\;
            Upload the active entries of
            $\Delta\mathbf B_i^t$ and
            $\Delta\mathbf A_i^t$\;
        }

        \tcp{Factor-wise aggregation}
        $\mathbf B^{t+1}\leftarrow\mathbf B^t+
        \sum_{i\in\mathcal S^t}\frac{w_i}{q}
        \Delta\mathbf B_i^t$\;
        $\mathbf A^{t+1}\leftarrow\mathbf A^t+
        \sum_{i\in\mathcal S^t}\frac{w_i}{q}
        \Delta\mathbf A_i^t$\;
    }

    \Return $(\mathbf B^T,\mathbf A^T)$\;
\end{algorithm}

\begin{algorithm}[t]
    \caption{\texttt{CFLoRA} with Heterogeneous Client Ranks}
    \label{alg:CFLoRA}
    \LinesNumbered
    \KwIn{$\mathbf W_0$, $(\mathbf B_{\mathrm{init}},
    \mathbf A_{\mathrm{init}})$, $R$,
    $\{r_i\}_{i=1}^{N}$, $\{w_i\}_{i=1}^{N}$,
    $m$, $T$, $K$, $\eta$,
    $p\in(0,1)$}
    \KwOut{$(\mathbf B^T,\mathbf A^T)$}

    $\mathbf B^0\leftarrow\mathbf B_{\mathrm{init}}$,
    $\mathbf A^0\leftarrow\mathbf A_{\mathrm{init}}$\;
    $q\leftarrow m/N$,
    $\eta_B\leftarrow\eta/p$,
    $\eta_A\leftarrow\eta/(1-p)$\;

    \For{$t\leftarrow 0$ \KwTo $T-1$}{
        Sample $m$ clients from $\{1,\ldots,N\}$\ to form $\mathcal S^t$;
        
        \tcp{Shared complementary mask}
        Draw $z_j^t\sim\operatorname{Bernoulli}(p)$
        independently for $j=1,\ldots,R$\;
        $\mathbf Z^t\leftarrow
        \operatorname{diag}(z_1^t,\ldots,z_R^t)$\;

        \tcp{Client-specific rank selections}
        \For{$i\in\mathcal S^t$}{
            % Sample $\mathcal S_i^t$ uniformly from the
            % $r_i$-element subsets of $\{1,\ldots,R\}$,
            % independently across participating clients\;
            % Construct $\mathbf E_i^t$ from the standard
            % basis vectors indexed by $\mathcal S_i^t$\;
            
            Uniformly sample $r_i$ distinct indices from $\{1,\ldots,R\}$\;
            Construct $\mathbf{E}_i^t$ from the standard basis vectors corresponding to the sampled indices\;
        }

        \For{$i\in\mathcal S^t$ \textnormal{ in parallel}}{
            $\widetilde{\mathbf B}_i^{t,0}
            \leftarrow\mathbf B^t\mathbf E_i^t$,
            $\widetilde{\mathbf A}_i^{t,0}
            \leftarrow(\mathbf E_i^t)^\top\mathbf A^t$\;
            $\widetilde{\mathbf Z}_i^t
            \leftarrow(\mathbf E_i^t)^\top
            \mathbf Z^t\mathbf E_i^t$\;

            \For{$k\leftarrow 0$ \KwTo $K-1$}{
                Compute gradients
                $(\mathbf g_{B,i}^{t,k},
                \mathbf g_{A,i}^{t,k})$ of
                $\ell\!\left(
                \mathbf W_0+
                \widetilde{\mathbf B}_i^{t,k}
                \widetilde{\mathbf A}_i^{t,k};
                \xi_i^{t,k}\right)$
                with respect to the compact factors\;
                $\widetilde{\mathbf B}_i^{t,k+1}
                \leftarrow\widetilde{\mathbf B}_i^{t,k}
                -\eta_B\mathbf g_{B,i}^{t,k}
                \widetilde{\mathbf Z}_i^t$\;
                $\widetilde{\mathbf A}_i^{t,k+1}
                \leftarrow\widetilde{\mathbf A}_i^{t,k}
                -\eta_A(\mathbf I_{r_i}
                -\widetilde{\mathbf Z}_i^t)
                \mathbf g_{A,i}^{t,k}$\;
            }

            $\Delta\widetilde{\mathbf B}_i^t
            \leftarrow\widetilde{\mathbf B}_i^{t,K}
            -\widetilde{\mathbf B}_i^{t,0}$\;
            $\Delta\widetilde{\mathbf A}_i^t
            \leftarrow\widetilde{\mathbf A}_i^{t,K}
            -\widetilde{\mathbf A}_i^{t,0}$\;
            Upload the active entries of
            $\Delta\widetilde{\mathbf B}_i^t$ and
            $\Delta\widetilde{\mathbf A}_i^t$\;
        }

        \tcp{Factor-wise aggregation}
        $\mathbf B^{t+1}\leftarrow\mathbf B^t+
        \sum_{i\in\mathcal S^t}\frac{w_i}{q}
        \Delta\widetilde{\mathbf B}_i^t
        (\mathbf E_i^t)^\top$\;
        $\mathbf A^{t+1}\leftarrow\mathbf A^t+
        \sum_{i\in\mathcal S^t}\frac{w_i}{q}
        \mathbf E_i^t
        \Delta\widetilde{\mathbf A}_i^t$\;
    }

    \Return $(\mathbf B^T,\mathbf A^T)$\;
\end{algorithm}

\clearpage

\section{Implementation and Complexity Analysis}
\label{app:implementation_complexity}

\setcounter{table}{0}
\renewcommand{\thetable}{\thesection.\arabic{table}}

\setcounter{figure}{0}
\renewcommand{\thefigure}{\thesection.\arabic{figure}}

\paragraph{Efficient Implementation.}
\texttt{CFLoRA} introduces \textit{negligible} additional communication overhead
beyond exchanging compact LoRA factors and updates. The additional control
information only involves $R$-bit complementary mask and
$\mathcal O(r_i\log R)$ bits specifying client $i$'s selected channel
indices. Even more, this information can alternatively be reproduced from shared random seeds. The main
additional operation is channel indexing: each participating client receives
the $r_i$ channels specified by $\mathbf E_i^t$, and the server maps the
returned increments back to their corresponding global positions.
These operations are implemented directly through indexing, without
explicitly constructing the selection matrices. Local training uses two
complementary branches to compute only the selected factor gradients,
while the server aggregates the compact increments with weights $w_i/q$.
Thus, the additional computation is limited to channel indexing and masking, which are negligible.

\begin{table}[h]
\centering
\caption{Normalized cost comparison at the same local rank for square projections between standard federated LoRA and our \texttt{CFLoRA}. MAC counts assume that input gradients are required and only
selected factor gradients are computed. Communication assumes compact
transmission at equal precision. Control information is negligible, which is excluded.}
\label{tab:analytical_complexity}
\small
\begin{tabular}{l|c|c}
\toprule
Quantity & Standard federated LoRA & \texttt{CFLoRA}\\
\midrule
Active adapter parameters       & $1$ & $1$\\
Trainable adapter parameters    & $1$ & $1/2$\\
Forward MACs                    & $1$ & $1$\\
Input-gradient MACs             & $1$ & $1$\\
Factor-gradient MACs            & $1$ & $1/2$\\
Total adapter MACs per step     & $1$ & $5/6$\\
Download payload per round      & $1$ & $1$\\
Upload payload per round        & $1$ & $1/2$\\
Combined communication per round & $1$ & $3/4$\\
\bottomrule
\end{tabular}
\end{table}

\begin{figure}[h]
    \centering
    % First subfigure
    \begin{subfigure}[b]{0.48\textwidth}
        \centering
        \includegraphics[width=\textwidth]{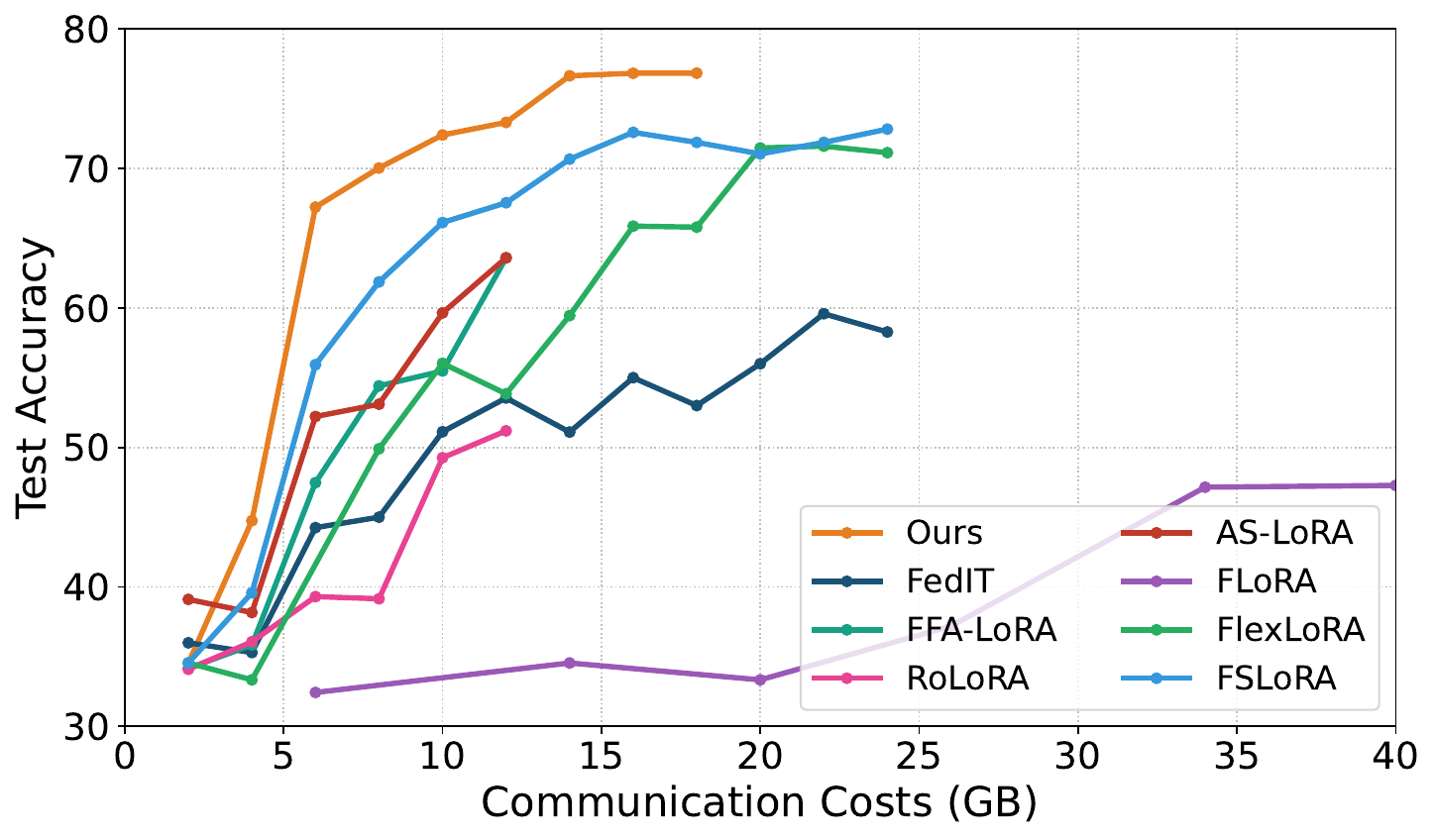}
        \caption{Test accuracy (matched) versus communication costs on the MNLI dataset.}
        \label{fig:CFLoRA_a}
    \end{subfigure}
    \hfill
    % Second subfigure
    \begin{subfigure}[b]{0.48\textwidth}
        \centering
        \includegraphics[width=\textwidth]{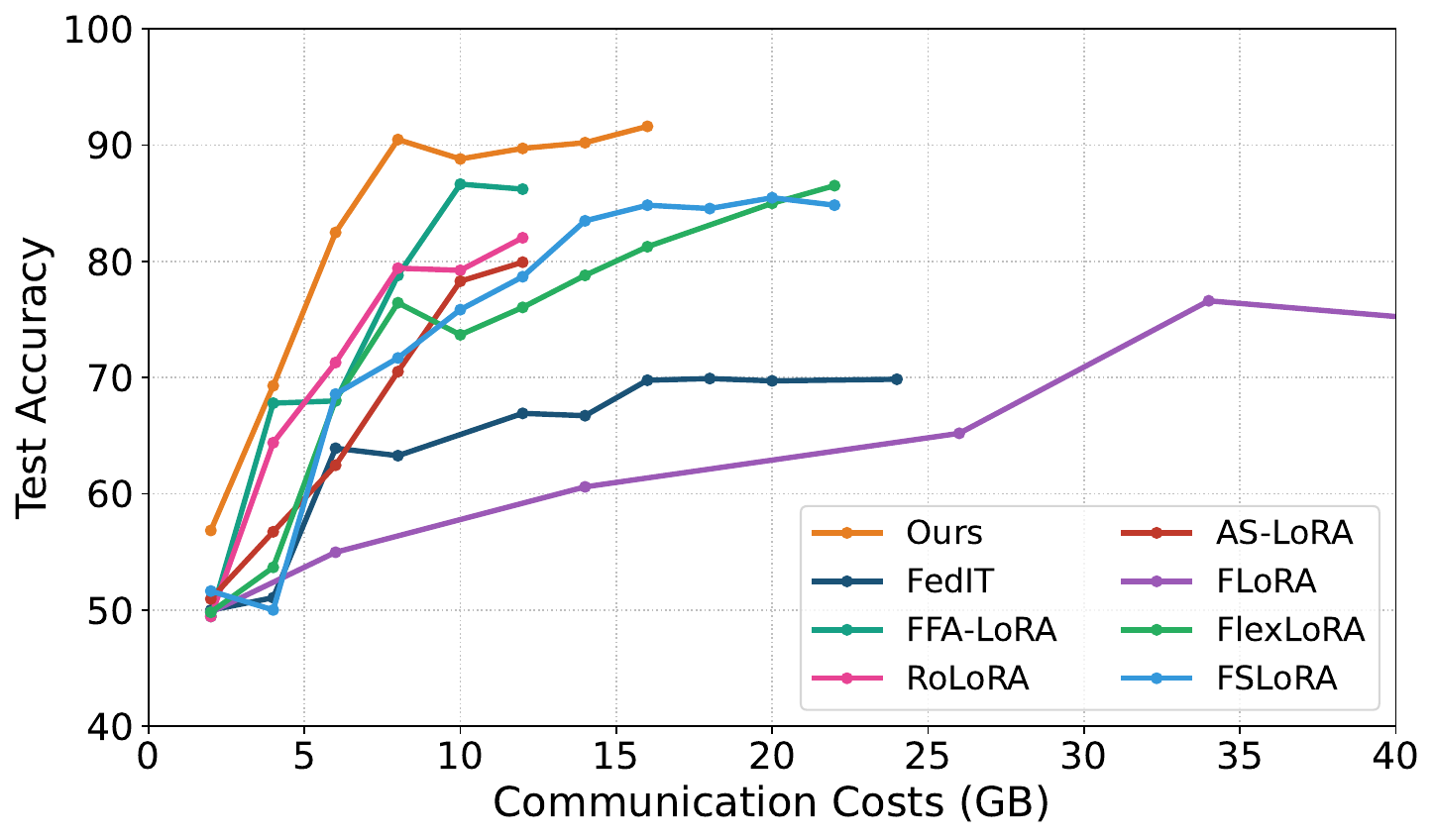} 
        \caption{Test accuracy versus communication costs on the QNLI dataset.}
        \label{fig:CFLoRA_b}
    \end{subfigure}
    
    \caption{Test accuracy versus communication costs of CFLoRA (Ours) and federated LoRA baselines on the MNLI and QNLI datasets.}
    \label{fig:Comm}
\end{figure}

\paragraph{Trainable Parameters.}
At the same local rank $r_i$, \texttt{CFLoRA} retains the same number
of adapter parameters as standard LoRA but updates only one factor
per rank channel: either a column of $\widetilde{\mathbf B}_i$
or the corresponding row of $\widetilde{\mathbf A}_i$.
For square projections (i.e., $d_{\mathrm{in}}=d_{\mathrm{out}}$),
these two vectors have equal size, so \texttt{CFLoRA} trains exactly
half of the adapter parameters, regardless of the realized
complementary mask.

The computational saving comes from omitting parameter-gradient
calculations for the frozen factor entries. Both compact factors
remain necessary for the forward pass.

\paragraph{Communication Costs.} Each participating
client downloads both compact adapter factors and uploads only
the updated columns of $\widetilde{\mathbf B}_i$ and rows of
$\widetilde{\mathbf A}_i$ in each round. At the same local rank,
the adapter download is unchanged compared with standard federated
LoRA, because both trainable and frozen entries are required
for the local forward pass. For square projections, only half
of the adapter entries are updated and uploaded. Therefore,
assuming the same transmission precision and excluding the
negligible control metadata, \texttt{CFLoRA} reduces the adapter
upload by $50\%$ and the combined adapter download and upload
by $25\%$. 

Table~\ref{tab:analytical_complexity} summarizes the comparison between standard federated LoRA and \texttt{CFLoRA}. Moreover, Figure~\ref{fig:Comm} illustrates accuracy versus communication costs under heterogeneous client ranks.

% ----------------------- Experiments -----------------------
% \newpage
\section{Experiments}
\label{app:experiments}
% Our evaluation will address four questions: (i) how CFLoRA compares with
% existing federated LoRA methods under homogeneous and heterogeneous client
% budgets; (ii) whether its structural savings translate into measured
% resource efficiency; (iii) how much adaptive allocation contributes beyond
% complementarity; and (iv) how the method responds to data heterogeneity,
% local computation, and client participation.

\subsection{Implementation Details}
\label{app:exp_details}
In the homogeneous setting, all clients share a uniform rank of $R=8$. In the heterogeneous setting, the server maintains rank $R=32$, and each client is assigned a fixed rank $r_i \in \{4, 8, 16, 32\}$ throughout training. To model heterogeneous client capabilities, we draw rank values from a normal distribution with mean $\mu=(a+b)/2$ and standard deviation $\sigma=(b-a)/6$, where $a=4$ and $b=32$. The sampled values are then clipped to $[4, 32]$ and mapped to the nearest allowed rank.
The LoRA factors $\mathbf{A}$ and $\mathbf{B}$ are initialized randomly and with zeros, respectively. For the shared global adapter, the LoRA scaling multiplier is fixed at $2$. In each round, all clients use the same complementary channel selection matrix $\mathbf{Z}$.
For the commonsense reasoning benchmark, the training corpus is a mixture of the ARC-c/e \citep{clark2018think}, BoolQ \citep{clark2019boolq}, HellaSwag \citep{zellers2019hellaswag}, OBQA \citep{mihaylov2018can}, PIQA \citep{bisk2020piqa}, and SIQA \citep{sap2019social} datasets. 
% With AdamW, the implementation applies compensation to the selected optimizer increments rather than merely rescaling gradients. 
Unless otherwise stated, the default hyperparameters used in our experiments are summarized in Table~\ref{tab:settings}.

\begin{table}[h]
\centering
\caption{parameter settings.}
\label{tab:settings}
\small
\begin{tabular}{l|c|c}
\toprule
Setting & RoBERTa / GLUE & LLaMA-3.2-3B-Instruct / commonsense reasoning \\
\midrule
Total clients & $25$  & $25$ \\
\midrule
Client participation ratio $q$ & $0.4$  & $0.4$ \\
\midrule
Communication rounds $T$ & $200$ & $100$\\
\midrule
Local steps $K$ & $50$ & $40$\\
\midrule
Batch size & $8$ & $8$\\
\midrule
Learning rate & $1\times10^{-4}$ & $2\times10^{-4}$\\
\midrule
LoRA dropout rate & $0.1$ & $0.1$\\
\midrule
Channel sampling probability $p$ & $0.9$ & $0.9$\\
% Independent seeds & $0,1,2,3$ & $0,1,2,3$\\
% Target module & ["query" "value"] & [“q proj”, “k proj”, “v proj”, “up proj”, “down proj”]\\
\bottomrule
\end{tabular}
\end{table}

% \newpage

\subsection{Experiments of LLaMA-3.2-3B-Instruct}
\label{app:llama_reasoning_hete}
In this section, we present an extended evaluation of fine-tuning LLaMA-3.2-3B-Instruct on the commonsense reasoning benchmark under both homogeneous and heterogeneous client ranks. The complete empirical results are summarized in Table~\ref{table:main_results_llama_homo_10of25} and Table~\ref{table:main_results_llma_reasoning_hete_10of25}.

\paragraph{Discussions on Homogeneous Client Ranks.}
As shown in Table~\ref{table:main_results_llama_homo_10of25}, \texttt{CFLoRA} achieves an overall average accuracy of $74.28\%$, outperforming all baseline methods including the strongest baseline FLoRA ($72.83\%$) and RoLoRA ($72.27\%$). Crucially, comparison with single-factor adaptation schemes reveals a fundamental empirical insight for large language models: methods that freeze one matrix factor (such as FFA-LoRA at $69.62\%$ and AS-LoRA at $69.71\%$) suffer a severe performance drop of over $4.5$ percentage points compared to \texttt{CFLoRA}. 

\paragraph{Discussions on Heterogeneous Client Ranks.}
Table~\ref{table:main_results_llma_reasoning_hete_10of25} demonstrates that the advantages of \texttt{CFLoRA} are pronounced in the presence of client rank heterogeneity ($r_i \in \{4, 8, 16, 32\}$). \texttt{CFLoRA} delivers an average accuracy of $75.36\%$, outperforming FSLoRA by $2.16$ points, FLoRA by $2.76$ points, FlexLoRA by $3.77$ points, and standard FedIT by $5.45$ points.

% ---------------------- Table Llama Homo 10-of-25 -----------------------
\begin{table*}[h]
\centering
% \resizebox{\textwidth}{!}{
% \setlength{\tabcolsep}{3mm}  % 保持原压缩
\caption{Testing accuracy for fine-tuning the LLaMA-3.2-3B-Instruct on the commonsense reasoning tasks under \textit{homogeneous} client ranks.}
\label{table:main_results_llama_homo_10of25} 
\resizebox{\textwidth}{!}{
\renewcommand{\arraystretch}{1}
\begin{tabular}{l|cccccc|c}
\toprule
\textbf{Method} & \textbf{ARC-c} & \textbf{ARC-e} & \textbf{BoolQ} & \textbf{OBQA} & \textbf{PIQA} & \textbf{SIQA} & \textbf{Avg.} \\
\midrule
FedIT & 67.72 {\scriptsize \color{gray} $\pm$ 0.45} & 84.48 {\scriptsize \color{gray} $\pm$ 1.10} & 62.44 {\scriptsize \color{gray} $\pm$ 0.27} & 74.53 {\scriptsize \color{gray} $\pm$ 0.47} & 70.51 {\scriptsize \color{gray} $\pm$ 2.39} & 66.47 {\scriptsize \color{gray} $\pm$ 1.89} & 71.03\\
FFA-LoRA & 67.45 {\scriptsize \color{gray} $\pm$ 0.10} & 80.62 {\scriptsize \color{gray} $\pm$ 0.29} & 64.17 {\scriptsize \color{gray} $\pm$ 0.04} & 63.37 {\scriptsize \color{gray} $\pm$ 1.93} & 70.39 {\scriptsize \color{gray} $\pm$ 3.89} & 71.69 {\scriptsize \color{gray} $\pm$ 0.54} & 69.62\\
RoLoRA & 69.60 {\scriptsize \color{gray} $\pm$ 1.93} & \textbf{86.78} {\scriptsize \color{gray} $\pm$ 0.12} & 64.18 {\scriptsize \color{gray} $\pm$ 0.07} & 73.73 {\scriptsize \color{gray} $\pm$ 2.45} & 69.61 {\scriptsize \color{gray} $\pm$ 3.88} & 69.73 {\scriptsize \color{gray} $\pm$ 2.27} & 72.27\\
AS-LoRA & 67.53 {\scriptsize \color{gray} $\pm$ 0.06} & 80.75 {\scriptsize \color{gray} $\pm$ 0.26} & 64.12 {\scriptsize \color{gray} $\pm$ 0.01} & 63.60 {\scriptsize \color{gray} $\pm$ 1.70} & 70.30 {\scriptsize \color{gray} $\pm$ 3.03} & 71.95 {\scriptsize \color{gray} $\pm$ 0.69} & 69.71\\
FLoRA & 70.65 {\scriptsize \color{gray} $\pm$ 0.16} & 83.67 {\scriptsize \color{gray} $\pm$ 0.03} & \textbf{65.26} {\scriptsize \color{gray} $\pm$ 1.10} & 72.07 {\scriptsize \color{gray} $\pm$ 0.75} & 72.67 {\scriptsize \color{gray} $\pm$ 1.34} & 72.67 {\scriptsize \color{gray} $\pm$ 1.41} & 72.83\\
\midrule
\textbf{CFLoRA} & \textbf{70.67} {\scriptsize \color{gray} $\pm$ 1.39} & 85.55 {\scriptsize \color{gray} $\pm$ 1.05} & 64.24 {\scriptsize \color{gray} $\pm$ 0.11} & \textbf{75.60} {\scriptsize \color{gray} $\pm$ 0.90} & \textbf{75.73} {\scriptsize \color{gray} $\pm$ 1.27} & \textbf{73.88} {\scriptsize \color{gray} $\pm$ 0.58} & \textbf{74.28} \\
\bottomrule
\end{tabular}
}
\end{table*}

\begin{table*}[h]
\centering
% \resizebox{\textwidth}{!}{
% \setlength{\tabcolsep}{3mm}  % 保持原压缩
\caption{Testing accuracy for fine-tuning the LLaMA-3.2-3B-Instruct on the commonsense reasoning benchmark under \textit{heterogeneous} client ranks.}
\label{table:main_results_llma_reasoning_hete_10of25}
\resizebox{\textwidth}{!}{
\renewcommand{\arraystretch}{1}
\begin{tabular}{l|cccccc|c}
\toprule
\textbf{Method} & \textbf{ARC-c} & \textbf{ARC-e} & \textbf{BoolQ} & \textbf{OBQA} & \textbf{PIQA} & \textbf{SIQA} & \textbf{Avg.} \\
\midrule
FedIT    & 67.04 {\scriptsize \color{gray} $\pm$ 0.32} & 80.13 {\scriptsize \color{gray} $\pm$ 0.97} & 62.37 {\scriptsize \color{gray} $\pm$ 0.81} & 71.33 {\scriptsize \color{gray} $\pm$ 1.40} & 70.24 {\scriptsize \color{gray} $\pm$ 2.13} & 68.32 {\scriptsize \color{gray} $\pm$ 1.37} & 69.91
\\
FLoRA & 66.98 {\scriptsize \color{gray} $\pm$ 1.20} & 84.18 {\scriptsize \color{gray} $\pm$ 0.42} & 63.55 {\scriptsize \color{gray} $\pm$ 1.94} & \textbf{76.00} {\scriptsize \color{gray} $\pm$ 0.57} & 74.27 {\scriptsize \color{gray} $\pm$ 1.63} & 70.63 {\scriptsize \color{gray} $\pm$ 1.88}  & 72.60
\\
FlexLoRA & 67.45 {\scriptsize \color{gray} $\pm$ 0.74} & 82.84 {\scriptsize \color{gray} $\pm$ 1.03} & 62.35 {\scriptsize \color{gray} $\pm$ 0.97} & 74.02 {\scriptsize \color{gray} $\pm$ 0.55} & 73.31 {\scriptsize \color{gray} $\pm$ 0.86} & 69.59 {\scriptsize \color{gray} $\pm$ 1.64} & 71.59
\\
FSLoRA  &  \textbf{69.88} {\scriptsize \color{gray} $\pm$ 0.27} & 84.51 {\scriptsize \color{gray} $\pm$ 0.49} & 64.21 {\scriptsize \color{gray} $\pm$ 0.75} & 73.80 {\scriptsize \color{gray} $\pm$ 0.92} & 74.96 {\scriptsize \color{gray} $\pm$ 1.50} & 71.83 {\scriptsize \color{gray} $\pm$ 1.69} & 73.20
\\
\midrule
\textbf{CFLoRA} & 69.51 {\scriptsize \color{gray} $\pm$ 0.12} & \textbf{85.82} {\scriptsize \color{gray} $\pm$ 0.77} & \textbf{66.51} {\scriptsize \color{gray} $\pm$ 0.35} & 
75.40 {\scriptsize \color{gray} $\pm$ 0.39}
&
\textbf{78.73} {\scriptsize \color{gray} $\pm$ 1.45} & \textbf{76.20} {\scriptsize \color{gray} $\pm$ 1.52}   & \textbf{75.36}
\\
\bottomrule
\end{tabular}
}
\end{table*}

% \newpage
% -------------------------- Ablation Study ----------------------------
\subsection{Further Experiments on Ablation Study on $p$} \label{app:ablation_study}
In this section, we provide additional ablation studies to further analyze the impact of the complementary channel selection probability $p$. In addition to Tables~\ref{table:ablation_robert_homo_10of25}, we evaluate the performance of CFLoRA across varying values of $p$ for RoBERTa on the GLUE benchmark under heterogeneous client ranks (Table \ref{table:ablation_robert_hete_10of25}), as well as for LLaMA-3.2-3B-Instruct on the commonsense reasoning benchmark under both homogeneous and heterogeneous client ranks (Table \ref{table:ablation_llama_homo_10of25} and \ref{table:ablation_llama_hete_10of25}). Both sets of experiments are conducted under a partial-participation setting with $q = 0.4$ to simulate practical resource-constrained federated learning environments.

% We investigate the impact of the complementary channel selection probability $p$ by varying it from $0.3$ to $1.0$. Note that setting $p=1.0$ restricts updates entirely to the $\mathbf B$ factor, reducing our method to FFA-LoRA. Table \ref{table:ablation_robert_homo_10of25} 
% summarizes the results across GLUE and commonsense reasoning benchmarks under both homogeneous and heterogeneous client ranks.
% Furthermore, the performance degradation at $p=1.0$ corroborates our core motivation: simultaneously co-adapting selected components of both $\mathbf A$ and $\mathbf B$ is critical for federated LoRA finetuning.

% ----------------- Table. Ablation study: Llama Homo 10-of-25 ------------------
\begin{table*}[h]
\centering
% \resizebox{\textwidth}{!}{
% \setlength{\tabcolsep}{3mm}  % 保持原压缩
\caption{Ablation study on selection probability $p$ for LLaMA-3.2-3B-Instruct on the commonsense reasoning benchmark under homogeneous client ranks.}
\label{table:ablation_llama_homo_10of25} 
\resizebox{\textwidth}{!}{
\renewcommand{\arraystretch}{1}
\begin{tabular}{l|cccccc|c}
\toprule
\textbf{Method} & \textbf{ARC-c} & \textbf{ARC-e} & \textbf{BoolQ} & \textbf{OBQA} & \textbf{PIQA} & \textbf{SIQA} & \textbf{Avg.} \\
\midrule
$p=0.3$ 
& 65.98 {\scriptsize \color{gray} $\pm$ 1.67} & 84.71 {\scriptsize \color{gray} $\pm$ 0.44} & 62.54 {\scriptsize \color{gray} $\pm$ 0.50} & 70.40 {\scriptsize\color{gray}$\pm$ 0.57} & 72.33 {\scriptsize\color{gray}$\pm$ 1.79} & 72.25 {\scriptsize\color{gray}$\pm$ 2.07} & 71.37 \\
$p=0.5$ 
& 69.51 {\scriptsize \color{gray} $\pm$ 0.96} & 84.87 {\scriptsize \color{gray} $\pm$ 0.15} & 63.62 {\scriptsize \color{gray} $\pm$ 0.67} & 74.63 {\scriptsize\color{gray}$\pm$ 0.65} & 74.53 {\scriptsize\color{gray}$\pm$ 1.48} & 73.26 {\scriptsize\color{gray}$\pm$ 0.62} & 73.40 \\
$p=0.7$
& 68.65 {\scriptsize \color{gray} $\pm$ 0.14} & \textbf{85.64} {\scriptsize \color{gray} $\pm$ 1.24} & 64.17 {\scriptsize \color{gray} $\pm$ 0.13} & 75.09 {\scriptsize\color{gray}$\pm$ 0.41} & \textbf{75.74} {\scriptsize\color{gray}$\pm$ 2.00} & 73.63 {\scriptsize\color{gray}$\pm$ 1.31} & 73.82 \\
$p=0.9$ 
& \textbf{70.67} {\scriptsize \color{gray} $\pm$ 1.39} & 85.55 {\scriptsize \color{gray} $\pm$ 1.05} & 64.24 {\scriptsize \color{gray} $\pm$ 0.11} & \textbf{75.60} {\scriptsize \color{gray} $\pm$ 0.90} & 75.73 {\scriptsize \color{gray} $\pm$ 1.27} & \textbf{73.88} {\scriptsize \color{gray} $\pm$ 0.58} & \textbf{74.28} \\
$p=1.0$ 
& 67.56 {\scriptsize \color{gray} $\pm$ 0.13} & 80.54 {\scriptsize \color{gray} $\pm$ 0.24} & \textbf{64.31} {\scriptsize \color{gray} $\pm$ 0.08} & 64.25 {\scriptsize \color{gray} $\pm$ 1.39} & 70.15 {\scriptsize \color{gray} $\pm$ 2.44} & 71.63 {\scriptsize \color{gray} $\pm$ 0.73} & 69.74 \\
\bottomrule
\end{tabular}
}
\end{table*}

% ----------------- Table. Ablation study: Llama Hete 10-of-25 ------------------
\begin{table*}[h]
\centering
% \resizebox{\textwidth}{!}{
% \setlength{\tabcolsep}{3mm}  % 保持原压缩
\caption{Ablation study on selection probability $p$ for LLaMA-3.2-3B-Instruct on the commonsense reasoning benchmark under heterogeneous client ranks.}
\label{table:ablation_llama_hete_10of25} 
\resizebox{\textwidth}{!}{
\renewcommand{\arraystretch}{1}
\begin{tabular}{l|cccccc|c}
\toprule
\textbf{Method} & \textbf{ARC-c} & \textbf{ARC-e} & \textbf{BoolQ} & \textbf{OBQA} & \textbf{PIQA} & \textbf{SIQA} & \textbf{Avg.} \\
\midrule
$p=0.3$ 
& 68.81 {\scriptsize \color{gray} $\pm$ 0.16} & 82.78 {\scriptsize \color{gray} $\pm$ 0.76} & 62.25 {\scriptsize \color{gray} $\pm$ 0.11} & 74.97 {\scriptsize\color{gray}$\pm$ 0.61} & 71.47 {\scriptsize\color{gray}$\pm$ 3.24} & 73.44 {\scriptsize\color{gray}$\pm$ 0.20} & 72.29 \\
$p=0.5$ 
& 68.89 {\scriptsize \color{gray} $\pm$ 0.45} & 83.76 {\scriptsize \color{gray} $\pm$ 0.71} & 65.25 {\scriptsize \color{gray} $\pm$ 0.44} & 75.50 {\scriptsize \color{gray} $\pm$ 1.13} & 76.93 {\scriptsize \color{gray} $\pm$ 0.46} & 75.61 {\scriptsize \color{gray} $\pm$ 0.07} & 74.32 
\\
$p=0.7$
& 69.10 {\scriptsize \color{gray} $\pm$ 0.60} & 84.53 {\scriptsize \color{gray} $\pm$ 0.66} & 66.18 {\scriptsize \color{gray} $\pm$ 0.50} & \textbf{75.82} {\scriptsize\color{gray}$\pm$ 0.57} & \textbf{78.88} {\scriptsize\color{gray}$\pm$ 4.90} & 76.01 {\scriptsize\color{gray}$\pm$ 0.17} & 75.09
\\
$p=0.9$ 
& \textbf{69.51} {\scriptsize \color{gray} $\pm$ 0.12} & \textbf{85.82} {\scriptsize \color{gray} $\pm$ 0.77} & \textbf{66.51} {\scriptsize \color{gray} $\pm$ 0.35} & 75.40 {\scriptsize \color{gray} $\pm$ 0.39} & 78.73 {\scriptsize \color{gray} $\pm$ 1.45} & \textbf{76.20} {\scriptsize \color{gray} $\pm$ 1.52} & \textbf{75.36}
\\
$p=1.0$ 
& 69.13 {\scriptsize \color{gray} $\pm$ 0.38} & 83.65 {\scriptsize \color{gray} $\pm$ 0.73} & 62.34 {\scriptsize \color{gray} $\pm$ 0.23} & 67.73 {\scriptsize \color{gray} $\pm$ 0.94} & 73.33 {\scriptsize \color{gray} $\pm$ 0.63} & 74.12 {\scriptsize \color{gray} $\pm$ 0.34} & 71.72 \\
\bottomrule
\end{tabular}
}
\end{table*}

\paragraph{Asymmetric Allocation.}
Theoretically, the convergence upper bounds established in Theorem~\ref{thm:convergence} and Corollary~\ref{cor:homogeneous} depend inversely on $p_{\min} = \min\{p, 1-p\}$ via the parameter $\kappa = p_{\min}^{-1}$. Under worst-case analysis where both factors are treated symmetrically, the theoretical bound is minimized at $p = 0.5$. However, our empirical findings reveal a distinct asymmetric behavior during actual federated optimization:
\begin{itemize}
    \item \textbf{Asymmetric Importance of Factor $\mathbf{B}$:} Increasing $p$ from $0.3$ to $0.9$ consistently improves downstream performance across all benchmark tasks. For instance, on RoBERTa under homogeneous ranks (Table~\ref{table:ablation_robert_homo_10of25}), the average score rises from $74.72$ at $p=0.3$ to $78.70$ at $p=0.5$, peaking at $79.73$ when $p=0.9$. A similar steady gain is observed in the heterogeneous setting (Table~\ref{table:ablation_robert_hete_10of25}), rising from $73.37$ ($p=0.3$) to $79.57$ ($p=0.9$). This behavior aligns closely with findings in centralized PEFT literature \citep{hayou2024lora+}, which demonstrate that the two low-rank factor matrices contribute asymmetrically to model representation updates: $\mathbf{B}$ typically requires larger effective updates or faster gradient movement to steer model representations, whereas $\mathbf{A}$ acts primarily as a random subspace projection. Prioritizing updates to $\mathbf{B}$ ($p=0.9$) thus accelerates optimization while preserving sufficient exploration along the latent channels.
    \item \textbf{Penalty for Starving $\mathbf{B}$ ($p=0.3$):} When $p$ is decreased to $0.3$ (meaning the majority of active channels update $\mathbf{A}$ instead of $\mathbf{B}$), we observe a pronounced drop of $5.01$ points on RoBERTa homogeneous and $6.20$ points on heterogeneous ranks. This confirms that allocating insufficient updates to $\mathbf{B}$ acts as a bottleneck for parameter adaptation.
\end{itemize}

\paragraph{Hyperparameter Robustness Across $p \in [0.5, 0.9]$.}
Despite the empirical preference for higher $p$, \texttt{CFLoRA} exhibits robustness across the wide range. On RoBERTa (both homogeneous and heterogeneous), variations across $p \in \{0.5, 0.7, 0.9\}$ remain within a tight margin of approximately $1.0$ point on average (e.g., $78.70 \to 78.84 \to 79.73$ in Table~\ref{table:ablation_robert_homo_10of25}, and $78.74 \to 79.45 \to 79.57$ in Table~\ref{table:ablation_robert_hete_10of25}). A similar plateau is evident on LLaMA-3.2-3B-Instruct in Table~\ref{table:ablation_llama_homo_10of25}, where tasks such as OBQA, PIQA, and SIQA perform consistently well from $p=0.5$ onward. This hyperparameter insensitivity is of substantial practical value in federated learning: practitioners do not need exhaustive tuning to identify a value of $p$.

% and reduces \texttt{CFLoRA} to the fixed-factor baseline FFA-LoRA
\paragraph{Catastrophic Degradation at Single-Factor Boundary ($p=1.0$).}
Crucially, when $p$ reaches $1.0$, which entirely eliminates updates to $\mathbf{A}$, performance experiences a sharp drop across different settings. On RoBERTa under homogeneous ranks, the overall average drops sharply by $4.39$ points (from $79.73$ at $p=0.9$ down to $75.34$ at $p=1.0$), with severe regressions across individual tasks such as CoLA ($-5.06$) and MNLI-M ($-6.80$). The degradation is even more pronounced when fine-tuning LLaMA-3.2-3B-Instruct on reasoning tasks (Table~\ref{table:ablation_llama_homo_10of25}). Comparing $p=0.9$ to $p=1.0$, accuracy drops $4.54\%$ on average. These empirical observations directly corroborate our core theoretical insights: freezing one low-rank projection matrix permanently severely curtails the adapter's expressive capacity.

% -------------------- Varying Client Participation --------------------
% \newpage
\subsection{Experiments with Varying Client Participation}
\label{app:full_participation}

We evaluate the effect of client participation by varying the participation rate $q$ from $0.2$ to $1.0$. Tables~\ref{table:main_results_robert_homo_varying_client} and \ref{table:main_results_robert_hete_varying_client} report test accuracy on the QNLI dataset under homogeneous and heterogeneous client ranks, respectively. Under homogeneous ranks, \texttt{CFLoRA} achieves the highest average accuracy, exceeding the strongest baseline by $1.38$ points. Under heterogeneous ranks, it improves the average over the strongest baseline by $1.52$ points. 
We further evaluate full participation across six GLUE tasks under homogeneous ranks in Table~\ref{table:main_results_robert_homo_full}, \texttt{CFLoRA} obtains the highest average score of $84.82$, surpassing the strongest baseline (FLoRA) by $2.26$ points. These results indicate that its performance advantage extends beyond the partial-participation setting used in the main experiments.

% ---------------------- Table GLUE Homo Varying Participation -----------------------
\begin{table*}[h]
\centering
% \resizebox{\textwidth}{!}{
% \setlength{\tabcolsep}{3mm}  % 保持原压缩
\caption{Testing accuracy for fine-tuning the RoBERTa model on the QNLI dataset under homogeneous client ranks with varying client participation.}
\label{table:main_results_robert_homo_varying_client}
\resizebox{\textwidth}{!}{
\renewcommand{\arraystretch}{1}
\begin{tabular}{l|ccccc|c}
\toprule
\textbf{Method} & \textbf{$q = 0.2$} & \textbf{$q = 0.4$} & \textbf{$q = 0.6$} & \textbf{$q = 0.8$} & \textbf{$q = 1.0$} &  \textbf{Avg.} \\
\midrule
FedIT    & 85.01 {\scriptsize \color{gray} $\pm$ 0.23} & 86.61 {\scriptsize \color{gray} $\pm$ 0.26} & 88.90 {\scriptsize \color{gray} $\pm$ 0.26} &  93.06 {\scriptsize \color{gray} $\pm$ 0.90} & 90.02 {\scriptsize \color{gray} $\pm$ 0.46} & 88.72 \\
FFA-LoRA & 86.13 {\scriptsize \color{gray} $\pm$ 0.35} & 88.49 {\scriptsize \color{gray} $\pm$ 0.23} & 90.94 {\scriptsize \color{gray} $\pm$ 0.33} & \textbf{94.09} {\scriptsize \color{gray} $\pm$ 0.32} & 91.42 {\scriptsize \color{gray} $\pm$ 0.19} & 90.21 \\
RoLoRA   & 86.52 {\scriptsize \color{gray} $\pm$ 0.27} & 87.60 {\scriptsize \color{gray} $\pm$ 2.21} & 92.60 {\scriptsize \color{gray} $\pm$ 0.46} & 93.23 {\scriptsize \color{gray} $\pm$ 0.14} & 93.10 {\scriptsize \color{gray} $\pm$ 0.46} & 90.61 \\
AS-LoRA  & 87.13 {\scriptsize \color{gray} $\pm$ 0.35} & 88.26 {\scriptsize \color{gray} $\pm$ 0.67} & 91.28 {\scriptsize \color{gray} $\pm$ 0.38} & 92.71 {\scriptsize \color{gray} $\pm$ 0.17} & 92.46 {\scriptsize \color{gray} $\pm$ 1.03} & 90.37 \\
FLoRA    & 88.83 {\scriptsize \color{gray} $\pm$ 0.24} & 89.27 {\scriptsize \color{gray} $\pm$ 0.20} & 90.69 {\scriptsize \color{gray} $\pm$ 0.41} & 92.93 {\scriptsize \color{gray} $\pm$ 0.44} & 91.41 {\scriptsize \color{gray} $\pm$ 0.35} & 90.63 \\
\midrule
\textbf{CFLoRA} & \textbf{88.95} {\scriptsize \color{gray} $\pm$ 0.06} & \textbf{90.61} {\scriptsize \color{gray} $\pm$ 0.25} & \textbf{92.98 }{\scriptsize \color{gray} $\pm$ 0.31} &  93.59 {\scriptsize \color{gray} $\pm$ 0.06} & \textbf{93.93} {\scriptsize \color{gray} $\pm$ 0.34} & \textbf{92.01} \\
\bottomrule
\end{tabular}
}
\end{table*}

% ---------------------- Table GLUE Hete Varying Participation -----------------------
\begin{table*}[h]
\centering
% \resizebox{\textwidth}{!}{
% \setlength{\tabcolsep}{3mm}  % 保持原压缩
\caption{Testing accuracy for fine-tuning the RoBERTa model on the QNLI dataset under heterogeneous client ranks with varying client participation.}
\label{table:main_results_robert_hete_varying_client}
\resizebox{\textwidth}{!}{
\renewcommand{\arraystretch}{1}
\begin{tabular}{l|ccccc|c}
\toprule
\textbf{Method} & \textbf{$q = 0.2$} & \textbf{$q = 0.4$} & \textbf{$q = 0.6$} & \textbf{$q = 0.8$} & \textbf{$q = 1.0$} &  \textbf{Avg.} \\
\midrule
FedIT & 68.01 {\scriptsize \color{gray} $\pm$ 4.58} & 70.34 {\scriptsize \color{gray} $\pm$ 2.01} & 85.89 {\scriptsize \color{gray} $\pm$ 0.37} & 90.03 {\scriptsize \color{gray} $\pm$ 3.26} & 88.22 {\scriptsize \color{gray} $\pm$ 2.36} & 80.50 \\
FLoRA   & 70.97 {\scriptsize \color{gray} $\pm$ 4.32} & 87.35 {\scriptsize \color{gray} $\pm$ 1.84} & 87.75 {\scriptsize \color{gray} $\pm$ 3.21} & 92.67 {\scriptsize \color{gray} $\pm$ 0.72} & 91.77 {\scriptsize \color{gray} $\pm$ 2.35} & 86.10 \\
FlexLoRA & \textbf{78.41} {\scriptsize \color{gray} $\pm$ 5.10} & 86.66 {\scriptsize \color{gray} $\pm$ 1.55} & 86.81 {\scriptsize \color{gray} $\pm$ 1.36} & 93.11 {\scriptsize \color{gray} $\pm$ 1.25} & \textbf{94.02} {\scriptsize \color{gray} $\pm$ 0.95} & 87.80 \\
FSLoRA   & 73.68 {\scriptsize \color{gray} $\pm$ 4.01} & 86.68 {\scriptsize \color{gray} $\pm$ 2.27} & 86.92 {\scriptsize \color{gray} $\pm$ 1.48} & 89.40 {\scriptsize \color{gray} $\pm$ 2.45} & 88.54 {\scriptsize \color{gray} $\pm$ 1.60} & 85.04 \\
\midrule
\textbf{CFLoRA} & 76.84 {\scriptsize \color{gray} $\pm$ 3.73} & \textbf{91.29} {\scriptsize \color{gray} $\pm$ 0.45} & \textbf{91.46} {\scriptsize \color{gray} $\pm$ 2.10} & \textbf{93.60} {\scriptsize \color{gray} $\pm$ 0.38} & 93.39 {\scriptsize \color{gray} $\pm$ 0.91} & \textbf{89.32} \\
\bottomrule
\end{tabular}
}
\end{table*}

% ---------------------- Table GLUE Homo Full Participation -----------------------
\begin{table*}[h]
\centering
% \resizebox{\textwidth}{!}{
% \setlength{\tabcolsep}{3mm}  % 保持原压缩
\caption{Testing accuracy for fine-tuning the RoBERTa model on the GLUE benchmark under homogeneous client ranks with full participation.}
\label{table:main_results_robert_homo_full}
\resizebox{\textwidth}{!}{
\renewcommand{\arraystretch}{1}
\begin{tabular}{l|cccccc|c}
\toprule
\textbf{Method} & \textbf{QNLI} & \textbf{MRPC} & \textbf{CoLA} & \textbf{MNLI-M} & \textbf{MNLI-MM} & \textbf{QQP} & \textbf{Avg.} \\
\midrule
FedIT    & 90.02 {\scriptsize \color{gray} $\pm$ 0.46} & 70.30 {\scriptsize \color{gray} $\pm$ 0.39} & 75.16 {\scriptsize \color{gray} $\pm$ 0.15} & 82.40 {\scriptsize \color{gray} $\pm$ 0.44} & 82.19 {\scriptsize \color{gray} $\pm$ 0.33} & 80.98 {\scriptsize \color{gray} $\pm$ 1.05}   & 80.18  \\
FFA-LoRA & 91.42 {\scriptsize \color{gray} $\pm$ 0.19} & 73.20 {\scriptsize \color{gray} $\pm$ 0.42} & 74.16 {\scriptsize \color{gray} $\pm$ 0.45} & 84.79 {\scriptsize \color{gray} $\pm$ 0.16} & 84.78 {\scriptsize \color{gray} $\pm$ 0.16} & 81.89 {\scriptsize \color{gray} $\pm$ 1.05}   & 81.71 \\
RoLoRA & 93.10 {\scriptsize \color{gray} $\pm$ 0.46} & 73.86 {\scriptsize \color{gray} $\pm$ 1.26} & 73.93 {\scriptsize \color{gray} $\pm$ 0.29} & 85.76 {\scriptsize \color{gray} $\pm$ 0.43} & 85.84 {\scriptsize \color{gray} $\pm$ 0.53} & 81.25 {\scriptsize \color{gray} $\pm$ 0.38}  & 82.29\\
AS-LoRA & 92.46 {\scriptsize \color{gray} $\pm$ 1.03} & 72.44 {\scriptsize \color{gray} $\pm$ 0.67} & 75.04 {\scriptsize \color{gray} $\pm$ 0.94} & 82.79 {\scriptsize \color{gray} $\pm$ 0.33} & \textbf{86.76} {\scriptsize \color{gray} $\pm$ 0.35} & 80.95 {\scriptsize \color{gray} $\pm$ 1.67} & 81.74\\
FLoRA      & 91.41 {\scriptsize \color{gray} $\pm$ 0.35}  & 76.34 {\scriptsize \color{gray} $\pm$ 0.80} & \textbf{77.26} {\scriptsize \color{gray} $\pm$ 0.25} & 83.66 {\scriptsize \color{gray} $\pm$ 0.77} & 83.88 {\scriptsize \color{gray} $\pm$ 0.66} & 82.79 {\scriptsize \color{gray} $\pm$ 0.32}  & 82.56 \\
\midrule
\textbf{CFLoRA} & \textbf{93.93} {\scriptsize \color{gray} $\pm$ 0.34} & \textbf{80.83} {\scriptsize \color{gray} $\pm$ 0.85} & 76.74 {\scriptsize \color{gray} $\pm$ 0.46} & \textbf{86.54} {\scriptsize \color{gray} $\pm$ 0.40} & 86.71 {\scriptsize \color{gray} $\pm$ 0.26} & \textbf{84.15} {\scriptsize \color{gray} $\pm$ 0.53}   & \textbf{84.82} \\
\bottomrule
\end{tabular}
}
\end{table*}

\end{document}